\documentclass[review]{elsarticle}

\usepackage{lineno,hyperref}
\modulolinenumbers[5]

\usepackage{color}

\usepackage{todonotes}

\usepackage{amsmath}
\usepackage{amssymb}
\newtheorem{theorem}{Theorem}
\newtheorem{proof}{Proof}
\newtheorem{definition}{Definition}

\journal{XXX}

\biboptions{authoryear}

\begin{document}

\begin{frontmatter}

\title{The division of labor in communication: Speakers help listeners account for asymmetries in visual perspective}


\author[psych]{Robert Hawkins \corref{mycorrespondingauthor}}
\author[psych]{Hyowon Gweon}
\author[psych,cs]{Noah Goodman}

\cortext[mycorrespondingauthor]{Corresponding author}

\address[psych]{Department of Psychology, Stanford University, Stanford, CA, US}
\address[cs]{Department of Computer Science, Stanford University, Stanford, CA, US}

\begin{abstract}
\small Recent debates over adults’ theory of mind use have been fueled by surprising failures of perspective-taking in communication, suggesting that perspective-taking can be relatively effortful.
How, then, should speakers and listeners allocate their resources to achieve successful communication? 
We begin with the observation that this shared goal induces a natural division of labor: the resources one agent chooses to allocate toward perspective-taking should depend on their expectations about the other's allocation.
We formalize this idea in a \emph{resource-rational} model augmenting recent probabilistic weighting accounts with a mechanism for (costly) control over the degree of perspective-taking. 
In a series of simulations, we first derive an intermediate degree of perspective weighting as an optimal tradeoff between expected costs and benefits of perspective-taking. 
We then present two behavioral experiments testing novel predictions of our model.
In Experiment 1, we manipulated the presence or absence of occlusions in a director-matcher task and found that speakers spontaneously produced more informative descriptions to account for ``known unknowns'' in their partner's private view. 
In Experiment 2, we compared the scripted utterances used by confederates in prior work with those produced in interactions with unscripted directors. 
We found that confederates were systematically less informative than listeners would initially expect given the presence of occlusions, but listeners used violations to adaptively make fewer errors over time.
Taken together, our work suggests that people are not simply ``mindblind''; they use contextually appropriate expectations to navigate the division of labor with their partner. 
We discuss how a resource rational framework may provide a more deeply explanatory foundation for understanding flexible perspective-taking under processing constraints.
\end{abstract}


\end{frontmatter}


\section{Introduction}

Our success as a social species depends on our ability to understand, and be understood by, different social partners across different contexts. 
\emph{Theory of mind}---the ability to represent and reason about others' mental states \citep{PremackWoodruff78_ChimpanzeeToM}---is considered to be the key cognitive mechanism that supports such context-sensitivity in our everyday social interactions. 
Being able to infer what others see, want, and think allows us to make more accurate predictions about their future behavior in different contexts and adjust our own behaviors accordingly. 
These inferences do not necessarily come for free, however.
Behavioral, developmental, and neural evidence increasingly suggests that at least some aspects of theory of mind use are computationally costly, requiring effortful processing under cognitive control \citetext{\citealp{saxe2006reading,brown2009role,nilsen2009relations,low2012implicit,ferguson2015task,bradford2015self,ryskin_perspective-taking_2015,symeonidou_development_2016,long2018individual,jouravlev2019tracking}, but see \citealp{rubio2019you}}.

How, then, should agents allocate their cognitive resources to successfully communicate with one another?
One prominent proposal is that agents cope with these constraints by using egocentric heuristics \citep{keysar1998egocentric, KeysarBarr___Brauner00_TakingPerspective, keysar_communication_2007, barr_perspective_2014}.
An `anchor-and-adjust' heuristic, in particular, allows agents to anchor on their own easily available perspective and effortfully adjust in the direction of another perspective to the extent that sufficient cognitive resources are available \citep{epley_perspective_2004}. 
Because the adjustment process satisfices at some threshold, heuristic accounts predict that optimal perspective-taking is rarely observed and communicative behavior is marked by some degree of egocentric bias.
These accounts have provided algorithmic explanations for a variety of key phenomena, such as the increase of egocentric biases under cognitive load and the effect of individual differences in working memory \citep{LinKeysarEpley10_ReflexivelyMindblind,roxbetanagel2000cognitive}.
However, they have also been challenged by apparently contradictory evidence. 
A number of subsequent eye-tracking studies suggested that people are sensitive to other perspectives from the earliest moments of processing, precisely when the egocentric bias is predicted to be the strongest \citep{nadig_evidence_2002, HellerGrodnerTanenhaus08_Perspective, BrownSchmidtTanenhaus08_TargetedGame, HannaTanenhausTrueswell03_CommonGroundPerspective}.

Alternative accounts have been proposed to address these issue. Under a \emph{simultaneous integration} account, for instance, listeners \citep{HellerParisienStevenson16_ProbabilisticWeighing} and speakers \citep{MozuraitisEtAl18_ProductionCombination} consider both their own private perspective and their partner's perspective at the same time \citep[for broader reviews of constraint-based theories, see][]{brown-schmidt_perspective_2018, degen_constraint-based_nodate}.
This account is formalized as a Bayesian probabilistic weighting model, where the degree to which each perspective contributes to the combination is given by a weighting parameter.
An intermediate value of this parameter, weighting each perspective about equally, has been found to account for prior results better than a purely egocentric or purely perspective-taking strategy. 
This proposal offers a computational-level explanation \citep{Marr10_Vision} for why prior eye-tracking studies have found early traces of the agent's own perspective \emph{and} their partner's.

Yet probabilistic weighting models also leave open an important question: 
\emph{Why} do people use the weighting they do in a given context?  What determines the degree to which people deviate from their egocentric perspective in different communicative scenarios? 
Without considering algorithmic-level processes, for example, it is difficult to explain what leads to apparently different weightings under cognitive load \citep{LinKeysarEpley10_ReflexivelyMindblind} or time constraints \citep{horton1996speakers}, or as a function of individual differences in working memory.
\cite{HellerParisienStevenson16_ProbabilisticWeighing} and \cite{MozuraitisEtAl18_ProductionCombination} discuss a potential role for the cognitive demands of inhibiting one's own perspective, but no explicit model has yet emerged that explains the flexible weighting of different perspectives in terms of more general principles of human cognition.\footnote{In technical terms, the weighting parameter has previously been treated as an ``exogeneous'' variable determined by factors outside the scope of the model. The problem of determining it as a function of other factors originating \emph{within} the model is known as ``endogenization'' \citep{mankiw2003macroeconomics}.}

\subsection{The division of labor in communication}

We argue in this paper for a \emph{resource-rational} account of perspective-taking in communication that formally fills this explanatory gap.
The recent development of resource-rational analysis  \citep{GriffithsLiederGoodman15_LevelsOfAnalysis, shenhav2017toward,lieder2019resource} has provided a framework for understanding a range of costly but important cognitive functions, including attention \citep{padmala2011reward}, working memory maintenance \citep{howes2016predicting}, planning \citep{callaway2018resource}, and decision-making under uncertainty \citep{lieder2018overrepresentation}, through the application of rational principles under cognitive constraints.
Computational-level accounts are often under-constrained: there are many solutions to the computational problem that could be considered equally ``optimal'' \emph{a priori} regardless of how costly or intractable the required computations are. 
Resource rational analyses attempt to place stronger constraints on these accounts by incorporating additional processing considerations.
The key insight, motivated by recent work on the mechanisms of cognitive control, is that agents consider both the functional value of a computation as well as its \emph{costs} \citep{ShenhavBotvinickCohen13_ControlACC,kool_mental_2018}, and behave in a way that is consistent with an approximately optimal tradeoff between them.
In other words, ``the question of interest has begun to shift from whether an individual is \emph{capable} of exerting cognitive effort to whether the individual will choose to do so'' \citep{kool2013intrinsic}.
This broader shift is consistent with recent mechanistic frameworks for language processing that argue for a central role of executive control and recurrent processing \citep{ferreira_mechanistic_2019}.

Communication presents a novel and interesting test case for resource rational analysis because it is a fundamentally cooperative, multi-agent activity. 
Participants in an interaction share the same joint goal, and their ability to achieve this goal depends on the \emph{joint effort} they each contribute. 
Collaboratively minimizing joint effort thus sets up a natural division of labor in communication \citep{tomasello_cultural_2009, Clark96_UsingLanguage}: the effort one participant ought to exert depends on how much effort they expect others to exert.
This mutual dependency poses a nontrivial representational and inferential challenge for participants.
We propose a resource rational formulation of this problem, which shares with simultaneous integration accounts the basic assumption that agents may be attending to and weighting their partner's perspective even at the outset of an interaction.
Indeed, as we show in Section 2, our proposal straightforwardly extends the family of probabilistic weighting models.
Unlike previous models, we provide an explicit computational explanation for how perspective weightings are set, in terms of a principled resource-rational trade-off between the expected costs and benefits of perspective-taking.
Our consideration of cost also addresses the algorithmic-level concerns that motivated egocentric heuristic models. 
However, rather than assuming agents are ``reflexively mindblind'' with no control over their egocentric biases, resource rationality predicts that agents can anticipate the perspective-taking needs of the interaction based on various contextual factors and make flexible decisions about the resources they dedicate toward perspective-taking.

We further suggest that the appropriate consideration of contextual factors can be derived from principles of Gricean reasoning  \citep{GoodmanFrank16_RSATiCS, FrankGoodman12_PragmaticReasoningLanguageGames,FrankeJager16_ProbabilisticPragmatics}.
A higher weighting of a partner's perspective may be expected to lead to gains in expected communicative accuracy even as it incurs a proportionally higher processing cost (i.e. in terms of cognitive resources allocated).
Critically, the expected gain in accuracy depends on pragmatic inferences about the other agent's underlying effort in the current context, and the overall cost may depend on environmental modulations such as cognitive load. 
Hence, this model is capable of systematic context- and partner-sensitivity in the effort an agent chooses to exert. 
In the following section, we analyze the specific Gricean considerations at play in the director-matcher task \citep{KeysarBarr___Brauner00_TakingPerspective}, highlighting the unique challenges facing the director. 
This Gricean analysis forms the basis of the expectations a listener ought to have about a speaker's behavior, thus informing the listener's decisions about the benefits of perspective-taking.

\subsection{Referring under uncertainty about the visual context}

The Gricean notion of cooperativity \citep{Grice75_LogicConversation} refers to the idea that speakers try to avoid saying things that are confusing or unnecessarily complicated given the current context, and that listeners expect this. 
For instance, imagine trying to help someone spot your dog at a busy dog park. 
It may be literally correct to call it a ``dog,'' but as a cooperative speaker you would understand that the listener would have trouble disambiguating the referent from many other dogs. 
Likewise, the listener would reasonably expect you to say something more informative than ``dog'' in this context. 
You may therefore prefer to use a more specific or \emph{informative} expressions, like ``the little terrier with the blue collar,'' even though it is more costly to produce \citep{BrennanClark96_ConceptualPactsConversation,VanDeemter16_ComputationalModelsOfReferring}. 
Importantly, you might also prefer more specific labels even when you yourself see only one dog at the moment.
For instance, in the presence of uncertainty about what the listener can see (e.g., when there may be other dogs from the listener's point of view), a cooperative speaker might want to be more specific to ensure that the listener identifies the correct dog. 

\begin{figure}[b!]
 \begin{center}
 \includegraphics[scale=.5]{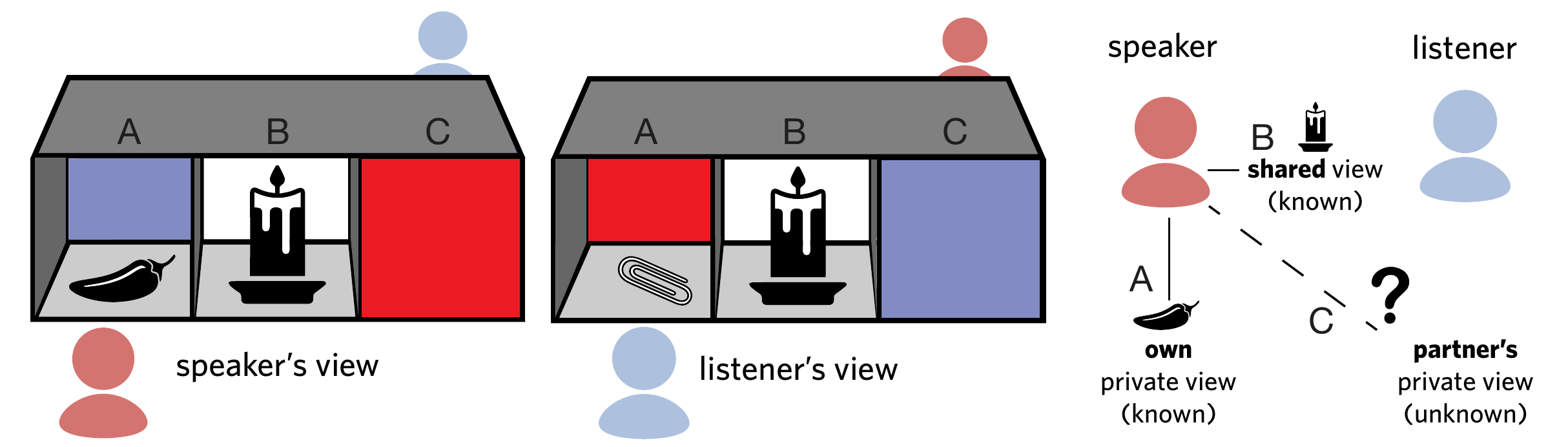}
 \caption{Schematic illustrating the three possible states that may be considered in a director-matcher task, where both parties may have objects in their own private view that are inaccessible to the other. In the presence of occlusions, agents must not only represent the known contents of their own private view (A) versus the content shared with their partner (B), but \emph{also} the unknown contents of their partner's private view (C). In practice, most studies place occlusions only on the speaker's side (red only) or only on the listener's side (blue only).}
 \label{fig:logic}
 \end{center}
 \end{figure}

While sensitivity to uncertainty about a partner's visual context is natural in everyday conversations, it has often been overlooked in the design of lab experiments.
We argue that the influential director-matcher paradigm \citep{KeysarBarr___Brauner00_TakingPerspective,KeysarLinBarr03_LimitsOnTheoryOfMindUse} places the speaker in an analogous situation to speakers at the dog park. 
In this task, a speaker instructs a listener to move objects around a grid.
Certain cells of the grid are covered to prevent the speaker from seeing some of the objects. 
It is therefore highly salient to the speaker that there exist hidden objects she herself cannot see but her partner can.
The speaker must generate a description such that a listener can identify the correct object among distractors, even though the speaker cannot be sure what all of the distractors are.

More generally, it is helpful to differentiate between three states that may in principle be considered by each agent in a director-matcher task: (A) the contents of one's own private view, which are known to oneself but not necessarily one's partner, (B) the contents of the shared view, which are known to both oneself and one's partner, and (C) the contents of the partner's private view, which are known to one's partner but not oneself (see Fig.~\ref{fig:logic} for an illustration).
For example, the version of the task introduced by \cite{KeysarBarr___Brauner00_TakingPerspective} only placed occluders on the speaker's side of the display and focused on the extent to which listeners distinguish between (A) and (B).
Because nothing was occluded from the listener's point of view (i.e. the display only used the red occluder from Fig. \ref{fig:logic}), the listener knew the speaker's private view (C) was the same as the shared view (B).
Extensive work has also examined how \emph{speakers} adjust their utterances (or fail to adjust their utterances) depending on their \emph{own} private information \cite[e.g.][]{nadig_evidence_2002, lane_dont_2006}, thus evaluating the extent to which the speaker accounts for differences between (A) and (B) in their production.
In this work, nothing was occluded from the speaker's point of view (i.e. the display only used the blue occluder in Fig. \ref{fig:logic}), so the speaker knew the listener's private view (C) was the same as the shared view (B).
Yet we still understand relatively little about the extent to which speakers naturally consider their own uncertainty about their \emph{partner's} private information (C), in scenarios like the one used by \cite{KeysarBarr___Brauner00_TakingPerspective}, where (C) is not identical to (B).
The possible objects behind the occluder are salient ``known unknowns'' that may influence a Gricean speaker's choice of referring expression, even if they have no private information of their own, i.e. even if (A) and (B) are identical.
Additionally, because prior work investigating listener perspective-taking commonly used confederates in the speaker role, it is possible that confederate behavior may have interacted with naive listener expectations of informativity in the presence of occlusions. 



 \subsection{The current work}
Our first goal is to derive and test Gricean predictions about how speakers should produce referring expressions under conditions of uncertainty about the listener's visual context.
As we show below, our model predicts that a speaker will compensate for her uncertainty about the listener's visual context by increasing the informativity of her utterance to some extent beyond what she would produce in a completely shared context.
In Experiment 1, we directly test this prediction by manipulating the presence and absence of occlusions in a simplified variant of the director-matcher task. 

Our second goal is to examine the consequences of this observation for the listener's allocation of effort.
The behavior observed in Experiment 1 establishes reasonable baseline expectations that listeners should use when deciding how much perspective-taking effort to allocate in the director-matcher task. 
In Experiment 2, we conduct a replication of the landmark study reported by \cite{KeysarLinBarr03_LimitsOnTheoryOfMindUse} with an additional \emph{unscripted} condition to evaluate the gap between the scripted referring expressions used by confederate speakers in prior work and what a naive speaker without a script would naturally say in the same interactive context \citep{KuhlenBrennan13_LanguageInDialogue, BavelasHealing13_MutualVisibility, TanenhausBrownSchmidt08_LanguageNatural}. 
Our model predicts that listeners will initially make more errors with confederate speakers (who are less informative than expected under a natural division of labor) compared with naive speakers. Critically, it also predicts that the gap will decrease over time; listeners in the confederate condition will gradually devote more effort to perspective-taking as they learn that the confederate is devoting less effort.

Taken together, this work aims to establish the plausibility of a resource rational basis for some degree of perspective-neglect on the part of both speaker and listener, and to emphasize the role of pragmatic expectations in determining this division of labor. 
It is important to note that our aim is to extend the explanatory power of recent probabilistic weighting models, not to falsify them. 
In fact, if we are successful in deriving from more basic principles the perspective-weighting proportions that were previously fit to empirical data, our model will necessarily make similar behavioral predictions for those experiments.
Consequently, our experiments were designed to expose and test the novel predictions of our extension, placing probabilistic weighting models on a firmer foundation, not necessarily to construct scenarios challenging the broader simultaneous integration view.
We clarify this theoretical relationship in the following section, and return to the broader implications and predictions of the resource rational view in the discussion.

\section{A resource-rational analysis of perspective-taking}

In this section, we formally derive the core predictions of our resource-rational analysis.
We begin with a brief review of the Rational Speech Act (RSA) framework, which formalizes pragmatic reasoning as recursive probabilistic inference, and define a new `ideal observer' model of perspective-taking under uncertainty about a partner's visual context. 
This model can then be mixed with an egocentric model, using the same probabilistic weighting mechanism proposed by \cite{HellerParisienStevenson16_ProbabilisticWeighing}.
Finally, we then conduct an analysis of the optimal parameter value for this mixture model given the additional assumption that there is higher cognitive cost to higher perspective-weighting. 

\subsection{Preliminaries}

The RSA framework derives language behavior from basic Gricean mechanisms of recursive social reasoning \citep{FrankGoodman12_PragmaticReasoningLanguageGames, GoodmanFrank16_RSATiCS, FrankeJager16_ProbabilisticPragmatics, KaoWuBergenGoodman14_NonliteralNumberWords, GoodmanStuhlmuller13_KnowledgeImplicature}.
In this framework, a pragmatic speaker $S$ is a decision-theoretic agent who must choose a referring expression $u$ to refer to a target object $o$ in a context $C$ by (soft)-maximizing a utility function $U$, capturing the tradeoff between the cost, or effort, of producing an utterance and the usefulness of each utterance for an imagined listener agent. 
%
%
%
%
In the context of the director-matcher task, the listener is a matcher who hears a referring expression $u$ in a context $C$ containing different objects and must select the target object $o$.
In principle, they do so by inverting their generative model of the speaker:
This formulation introduces a mutually recursive dependency between the speaker and listener.
A key idea of the RSA framework is to introduce a ``base case'' for this recursion to bottom out: specifically, we define a ``literal listener" $L_0$ who updates their beliefs about which object is the target of reference using the literal meaning of the utterance, $\mathcal{L}(u,o)$.
In our referential context, $\mathcal{L}$ simply represents a simple lexical semantics for $u$: If $u$ is true of $o$ (i.e. if $u$ is ``square'' and $o$ is actually a square) then $\mathcal{L}(o,u) = 1$; otherwise, $\mathcal{L}(o,u) = 0$.
The literal listener then serves as the foundation for a chain of additional layers of recursive reasoning:
\begin{equation}
\begin{array}{rcl}
P_{L_0}(o | u, C) & \propto & \mathcal{L}(o,u) P(o) \\
P_{S_1}(u | o, C) & \propto & \exp\{\alpha  \log P_{L_0}(o | u, C) - \textrm{cost}(u)\}\\
P_{L_1}(o | u, C) & \propto & P_{S_1}(u | o, C) P(o)
\end{array}
\end{equation}
where normalization takes place over objects $o \in C$ or utterances $u \in \mathcal{U}$.

\subsection{Reasoning about asymmetries in visual access}

This basic setup assumes that the speaker reasons about a listener sharing the full context $C$ in common ground, i.e. that the entire display is in state (B) of Fig. \ref{fig:logic}.
But how does a speaker refer to a target object when they know their partner has additional, unknown distractor objects in their private view, as in the scenario from \cite{KeysarBarr___Brauner00_TakingPerspective}?
Models which contrast the egocentric domain of reference against what is shared in common ground would predict no difference in speaker production between this scenario and one with no occlusions at all.
After all, because the speaker is not shown any private information, the information in the speaker's egocentric perspective, state (A), is equivalent to the information they know to be in common ground, state (B): all visible objects in the speaker's view are also clearly visible to the listener.
The relevant perspective at issue for evaluating the speaker's perspective-taking in this scenario is not the content of the shared view, but instead the (unknown) private contents of the \emph{listener's} visual field, state (C).

In the RSA framework, speaker uncertainty about the listener's visual field is represented straightforwardly by a probability distribution: for example, \cite{GoodmanStuhlmuller13_KnowledgeImplicature} examined a case where the speaker has limited perceptual access to the objects they are describing, and predicted how a pragmatic listener who is taking the speaker's perspective should interpret the speaker's utterances in light of such access.
In the case of the specific director-matcher task studied in this paper, the state of the world is the space of objects $\mathcal{O}$ seen by one's partner. 
Because the speaker knows that objects may be behind occluders, we introduce uncertainty $P(o_h)$ over which object $o_h \in \mathcal{O}$, if any, is hidden behind each occlusion.
The speaker ought to then marginalize over these alternatives when reasoning about which object a literal listener will select from the set of objects in their view. 
This gives us a speaker utility under conditions of \emph{asymmetries in visual access}:
\begin{equation}
\label{eq:asym}
U_{S_1}^{asym}(u; o, C) =\sum_{o_h \in \mathcal{O}} P(o_h) \log  P_{L_0}(o | u, C \cup o_h) - \textrm{cost}(u)
\end{equation} 
where $C$ still denotes the set of objects that the agent knows to be in common ground. 
Conversely, we can define an \emph{egocentric} speaker who ignores the possible existence of hidden objects that only the listener can see and only seeks to be informative relative to the objects in their own view (which, again, happens to be identical to the common ground):
\begin{equation}
\label{eq:ego}
U_{S_1}^{ego}(u; o, C) =  \log  P_{L_0}(o | u, C) - \textrm{cost}(u)
\end{equation} 

The analogous asymmetric and egocentric models for the listener are more straightforward: they have full information about exactly which objects are in each person's view because nothing is occluded from their own view.
\begin{equation}
\label{eq:listenerasymego}
\begin{array}{rcl}
P_{L_i}^{asym}(u;o,C) & = & P_{L_i}(u; o, C - \{o_h\}) \\
P_{L_i}^{ego}(u;o,C) & = & P_{L_i}(u; o, C)
\end{array}
\end{equation} 

\subsection{A probabilistic weighting model}

The utility in Eq. \ref{eq:asym} represents an ``ideal'' perspective-taking speaker, while the utility in Eq. \ref{eq:ego} represents a completely egocentric speaker.
Next, we follow \cite{HellerParisienStevenson16_ProbabilisticWeighing} in allowing for a probabilistic mixture between these two perspectives using an interpolation weight $w_S \in [0,1]$: 
\begin{equation}
\label{eq:speakermix}
U_{S_1}^{mix}(u; o, C, w_S) = w_S \cdot U_{S_1}^{asym}(u; o, C) + (1 - w_S) U_{S_1}^{ego}(u; o, C)
\end{equation}
When $w_S = 0$, the speaker using this utility is purely `occlusion-blind' or egocentric: she assumes her partner sees exactly the same objects she herself does\footnote{Note that this could correspond to either an `egocentric' domain of reference or a `common ground' domain, which are equivalent in the classic variant of the director-matcher task we are considering.}.
When $w_S = 1$, this speaker is purely `occlusion-sensitive': she assumes there may be additional objects in her partner's view that she cannot see behind the occlusions.
Similarly, we define a mixture model for the listener, with $w_L = 0$ corresponding to the purely egocentric domain and $w_L=1$ corresponding to the objects in common ground (i.e. the speaker's perspective):
\begin{equation}
\label{eq:listenermix}
P_{L_i}^{mix}(o; u, C, w_L) \propto w_L \cdot P_{L_i}^{asym}(u; o, C) + (1 - w_L) \cdot P_{L_i}^{ego}(u; o, C) 
\end{equation}

A critical point of difference between \cite{HellerParisienStevenson16_ProbabilisticWeighing} and our recursive RSA model formulation, however, is that we assume the occlusion-aware speakers and listeners account for the fact that their partner is also a mixture model with some mixture weight, i.e. we revise  Eq. \ref{eq:asym} and  Eq. \ref{eq:listenerasymego} as follows:
\begin{equation}
\begin{array}{rcl}
U^{asym}_{S_1}(o; u, C, w_L) & = & \sum_{o_h \in \mathcal{O}} P(o_h) \log  P_{L_0}^{mix}(o | u, C \cup o_h, w_L) - \textrm{cost}(u) \\
P^{asym}_{L_1}(o; u, C, w_S) & \propto & P_{S_1}^{mix}(u | o, C - \{o_h\}, w_S) P(o)
\end{array}
\end{equation}
and update Eq. \ref{eq:speakermix} and Eq. \ref{eq:listenermix} to pass this parameter through:
\begin{equation}
\begin{array}{rcl}
U_{S_1}^{mix}(u; o, C, w_S, w_L) & = & w_S \cdot U_{S_1}^{asym}(u; o, C, w_L) + (1 - w_S) U_{S_1}^{ego}(u; o, C) \\
P_{L_i}^{mix}(o; u, C, w_L, w_S) &  \propto & w_L \cdot P_{L_i}^{asym}(u; o, C, w_S) + (1 - w_L) \cdot P_{L_i}^{ego}(u; o, C) 
\end{array}
\end{equation}
Further, we assume agents have uncertainty about the exact weight their partner is using, and marginalize over it when choosing an action.
In this way, we obtain the theoretical dependency between mixture weights that is characteristic of a division of labor: one agent's behavior at a particular mixture weight setting will differ depending on the mixture weight they think their partner is using. 
The final models are given as follows:
\begin{equation}
\begin{array}{rcl}
\label{eq:finaleqs}
P_{L_0}(o; u, C, w_L) & \propto & P_{L_0}^{mix}(o; u, C, w_L) \\
P_{S_1}(u; o, C, w_S) & \propto & \int_{w_L}P(w_L)\exp\{\alpha U^{mix}_{S_1}(u; o, C, w_S, w_L)\} P(o)\, \mathrm{d}w_L \\
P_{L_1}(o; u, C, w_L) & \propto & \int_{w_S}P(w_S)P_{L_1}^{mix}(o;u,C,w_L,w_S) \,\mathrm{d} w_S 
\end{array}
\end{equation}

To build intuition about the behavior of these models, it is useful to consider a few example cases.
First, consider the behavior of the literal listener at the extreme values of $w_L$: when $w_L$ is close to 1 and the listener is fully considering the speaker's perspective, it will never select an occluded object, even if it has exactly the same attributes as the target in common ground. When $w_L$ is close to 0, it will select the occluded object exactly half of the time if it matches the literal description. 
Intermediate values of $w_L$ interpolate between these cases, leading to lower but non-zero probability of selecting the occluded object.

Now, consider the behavior of a pragmatic speaker model reasoning about this literal listener and trying to decide which utterance to produce. If the speaker's mixture weight $w_S$ is close to 0, then it doesn't consider possible existence of occluded objects and produces a description that is sufficient to disambiguate the target from alternatives only in its own view. If $w_S$ is close to 1 then the speaker's decision depends purely on the \emph{mixture weight the literal listener is expected to be using}.
When $w_L=1$, the listener will always pick the object matching the description in the speaker's view, no matter how minimal a description is given, so there is no benefit to producing a more detailed but costly utterance. 
Conversely, when $w_L=0$, then shorter utterances are risky: there are more possible hidden objects $o_h$ that would match a shorter description. 
Every additional feature the speaker mentions helps guard against a broader class of potential hidden objects, so it may be worth incurring the additional production cost to add information (see Appendix A for a more extensive proof of this behavior).
When the speaker marginalizes over their prior expectations about the value of $w_L$, these behaviors are combined: the speaker model errs on the side of more informative utterances, to hedge against the risks of lower values of $w_L$ and confusing hidden objects $o_h$.

\subsection{Resource-rational analysis}

  \begin{figure*}[t!]
 \begin{center}
 \includegraphics[scale=.7]{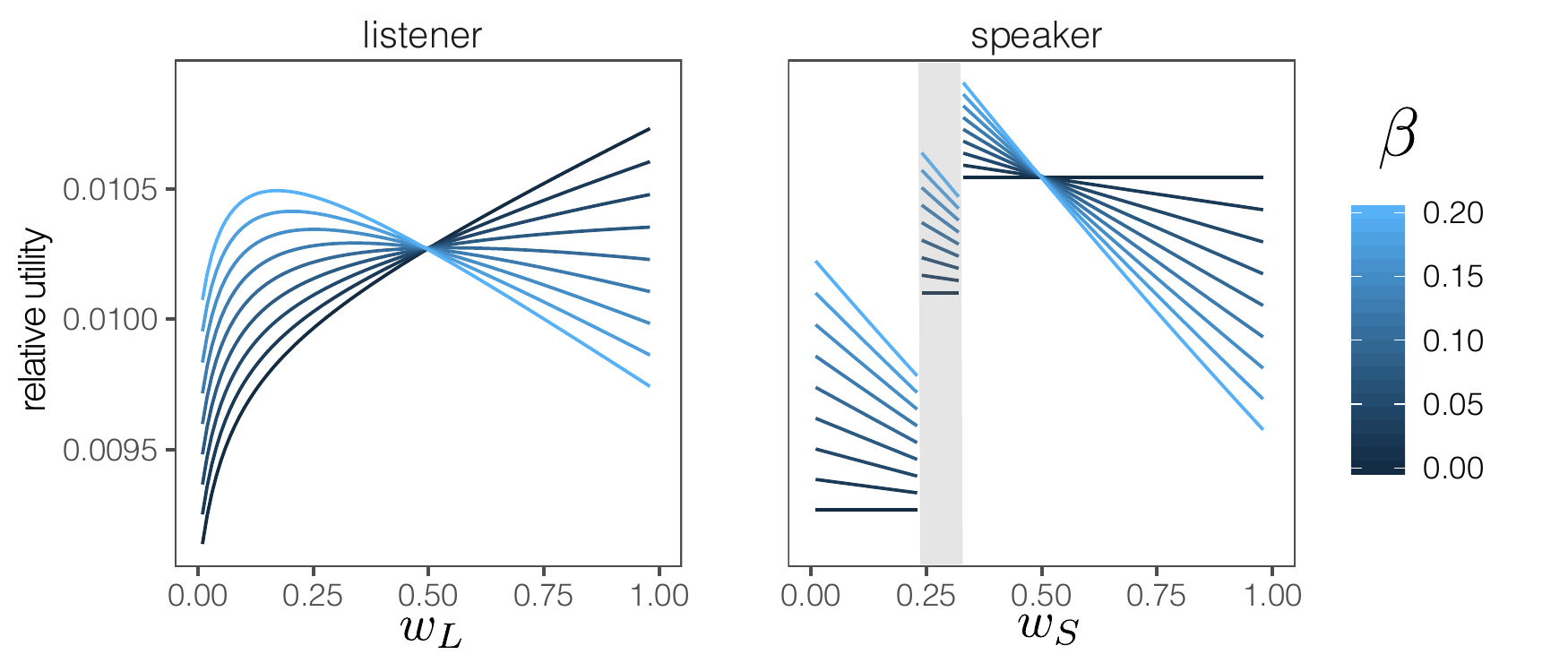}
 \caption{Results of resource-rational analysis of speaker and listener models. Above a certain value of $\beta$ (i.e. if perspective-taking is sufficiently effortful) then an intermediate weighting of perspective-taking is boundedly optimal. The discontinuities in the speaker plot occur when a higher level of perspective-taking motivates the speaker to switch to a longer utterance (e.g. ``the blue square'' instead of ``the square'' at $w_S=0.235$, followed by ``the blue checked square'' at $w_S=0.325$.)}
 \label{fig:speaker_sim}
 \end{center}
 \end{figure*}
 
We now conduct a resource-rational analysis of these mixture models to find the optimal weight under the cost and benefits of dedicating increasing cognitive resources to perspective-taking.
We begin by considering the resource-rational tradeoff between (1) the expected value of communicative accuracy and (2) the cognitive cost of perspective-taking.
We define the former value as the expected probability of the listener choosing the true target.
At each level of speaker perspective-taking $w_S$, the speaker agent will prefer some utterance $u^*$; they can then compute the probability of $L_0$ selecting the target after hearing this utterance.
Similarly, at each level of listener perspective-taking $w_L$, the listener agent will have some likelihood of selecting the target upon hearing the different speaker utterances. 
In both cases, the agents have uncertainty about their partner's level of perspective-taking and must therefore compute expected accuracy by marginalizing over the weight prior. 

If communicative accuracy were the only consideration, it would always be preferable to use maximal perspective-taking (i.e. $w_S=w_L=1$), since higher perspective-taking leads to higher accuracy on average.
In a resource-rational model, however, these benefits are traded off against the costs of perspective-taking.
For simplicity, we assume that cost is linear in the degree of perspective-taking; It's unclear whether there exists a perspective-taking algorithm where this linearity holds exactly, and our analysis holds under the weaker condition that cost is strictly increasing.
We use $\beta$ to denote the slope of this linear cost term.

Our analysis proceeds by running the $S_1$ and $L_1$ models in Eq. \ref{eq:finaleqs} with different choices of $w_S$ and $w_L$, respectively. 
In cognitive terms, this corresponds to an introspective speaker and listener meta-cognitively simulating the costs and benefits of exerting each amount of perspective-taking effort. 
\begin{equation}
\begin{array}{rcl}
U_{S_{RR}}(w_S) & = & \mathbb{E}_{P(w_L)}[P_{L_0}(o; u^*, C, w_L)] - \beta \cdot w_S \\
U_{L_{RR}}(w_L) & = & \mathbb{E}_{P(w_S)}[P_{L_1}(o; u^*, C, w_L)] - \beta \cdot w_L
\end{array}
\end{equation}
where in both cases $u^*$ is the utterance produced by the speaker model using weight $w_S$:
$$u^* =\textrm{argmax}_u P_{S_1}(u; o, C, w_S)$$

We define the optimal weights as the arguments for which this utility is maximized:
\begin{equation}
\begin{array}{rcl}
w_S* & = & \textrm{argmax}_{w_S}\left[U_{S_{RR}}(w_S) \right] \\
w_L* & = & \textrm{argmax}_{w_L}\left[U_{L_{RR}}(w_L)\right]
\end{array}
\end{equation}

To derive concrete simulation results, we set $\alpha=5$ and $\textrm{cost}(u) = 0.01$ for all $u$, and sweep over different values of $\beta$.
The utterance space, object space, and context $C$ are based on the ones we use below in Experiment 1: objects varied in shape, color, and texture, and the speaker model was able to produce any combination of shape, color, and texture descriptors. 
To simplify analytic enumeration over these spaces, we set the target to be a particular setting of features (i.e. ``color 1, texture 1, shape 1''), and represented other objects and utterances in terms of whether they match the target on each dimension (e.g. ``same color, different texture, different shape'').
We used uniform priors over the identity of the (single) hidden object, $P(o_h)$ and when taking internal expectations over $w_S$ and $w_L$. 

Results of this analysis are shown in Fig. \ref{fig:speaker_sim}.
As suggested above, when there is no cost to perspective-taking (i.e. $\beta=0$), the expected likelihood of communicative success increases monotonically as a function of perspective-taking weight.
Once we factor in a degree of cost for higher perspective taking, however, the increased likelihood of communicative success at higher weights begins to be offset by the corresponding increase in effort required to achieve it.
Above a certain $\beta$, we find that an intermediate perspective weighting is optimal for both speaker and listener.
That is, once perspective taking has a certain cost, a resource-rational agent will weight their partners' perspective to a lesser extent.
For instance, at $\beta=0.1$, we find that the optimal speaker weight is $w_S^* =0.33$ and the optimal listener weight is $w_L^* =0.55$.
At higher values of $\beta$, the optimal weighting drops for the listener. 
This simulation reveals the explanatory logic of the resource-rational framework. 
We showed conditions under which the intermediate probabilistic weightings empirically measured by \cite{HellerParisienStevenson16_ProbabilisticWeighing} and \cite{MozuraitisEtAl18_ProductionCombination} emerge from underlying computational principles: specifically, the tradeoff between the costs and benefits of different degrees of perspective-taking.

\subsection{Two qualitative predictions}

We highlight two key predictions of this formulation which motivate our experiments.
First, our proposal for a basic asymmetric speaker utility in by Eq. \ref{eq:asym} already leads to a novel prediction about speaker behavior in the presence of `known unknowns' hidden by occlusions.
This formulation goes beyond the speaker model of \cite{MozuraitisEtAl18_ProductionCombination}, which only considers the case where the speaker has perfect knowledge of the mismatch between their own private information and the listener's private information.
Specifically, as we show analytically in Appendix A, our model qualitatively predicts that speakers will anticipate possible confusion from the listener's perspective, and produce additional information beyond what would be necessary from their own viewpoint.
Note that such additional information would be unnecessary if the listener were expected to use perfect perspective-taking (i.e. if the the speaker believed $w_L = 1$); the functional need to increase informativity arises only when speakers assign nonzero probability to the possibility that listeners would act egocentrically.
This prediction is not strictly a consequence of the speaker's own resource-rational tradeoff (it is expected to emerge to some degree at any $w_S > 0$); however, it is a foundational assumption on which the rest of our resource-rational modeling rests and is therefore the first target of our empirical investigation in Experiment 1.

Second, a key prediction distinguishing the resource-rational framework from a ``fixed capacity'' egocentric heuristic model is that agents may flexibly adjust the effort dedicated to perspective-taking depending on contextual factors.
The optimal level of perspective-taking for one agent depends on reasoning about expected communicative success.
Expected success, in turn, depends on the perspective-taking weight being used by the other agent. 
Both agents bring into the interaction some prior expectations about this weight, but by comparing their partner's behavior to what would be expected at different levels of perspective-taking, they can update these beliefs. 
These updated beliefs lead to different expectations of future communicative success and may therefore shift the optimal level of their own perspective taking. 
In other words, our model predicts that agents will adapt their own perspective taking effort to their partner's to maintain a resource-rational tradeoff. 

We suggest that these mechanisms may help shed further light on the errors made by listeners (matchers) in \cite{KeysarLinBarr03_LimitsOnTheoryOfMindUse}. 
Specifically, the scripted referring expressions produced by confederate speakers in the director role may have been \emph{less informative} than what listeners in the matcher role would naturally expect from a cooperative speaker, leading to an initially mis-calibrated level of listener perspective-taking.
In Appendix B, we simulate a resource-rational listener agent playing a director-matcher task with a speaker who systematically produces less informative utterances than expected under the prior.
As expected, we find that the listener model gradually increases their own perspective-taking weight as they make stronger inferences about their partner's effort. 
In Experiment 2, we test this prediction in two ways.
First, we evaluate the actual gap between natural speaker behavior and confederate speaker behavior.
Second, we evaluate the extent to which listeners adapt over subsequent rounds.

\section{Experiment~1: Speaker production under uncertainty about the listener's visual context}

While the director-matcher scenario was originally designed to focus on the effort required of the listener (who must think about which cells in their own view are visible from the speaker's view), our model highlights that the same occlusions also demand theory of mind use, \emph{vis a vis} pragmatic audience design, on the part of the speaker.
The speaker must anticipate what level of informativity would be appropriate given the possibility of hidden distractors that are visible only to the listener.
To test this novel prediction of our asymmetric speaker model, we designed a simplified version of the director-matcher task that allows us to causally isolate the effect of occlusions on production.
Note that this task is not designed to ask whether speakers produce perfectly ``optimal'' referring expressions by some absolute standard --- it is implausible that they would know the true underlying distribution of hidden objects within the context of this task, and as our model formalizes, they would face their own resource constraints even if they did. 
Instead, our prediction is qualitative: do speakers spontaneously produce more informative referring expressions in the presence of occlusions than they do in the absence of occlusions?

\subsection{Methods}

  \begin{figure}[t!]
 \begin{center}
 \includegraphics[scale=.4]{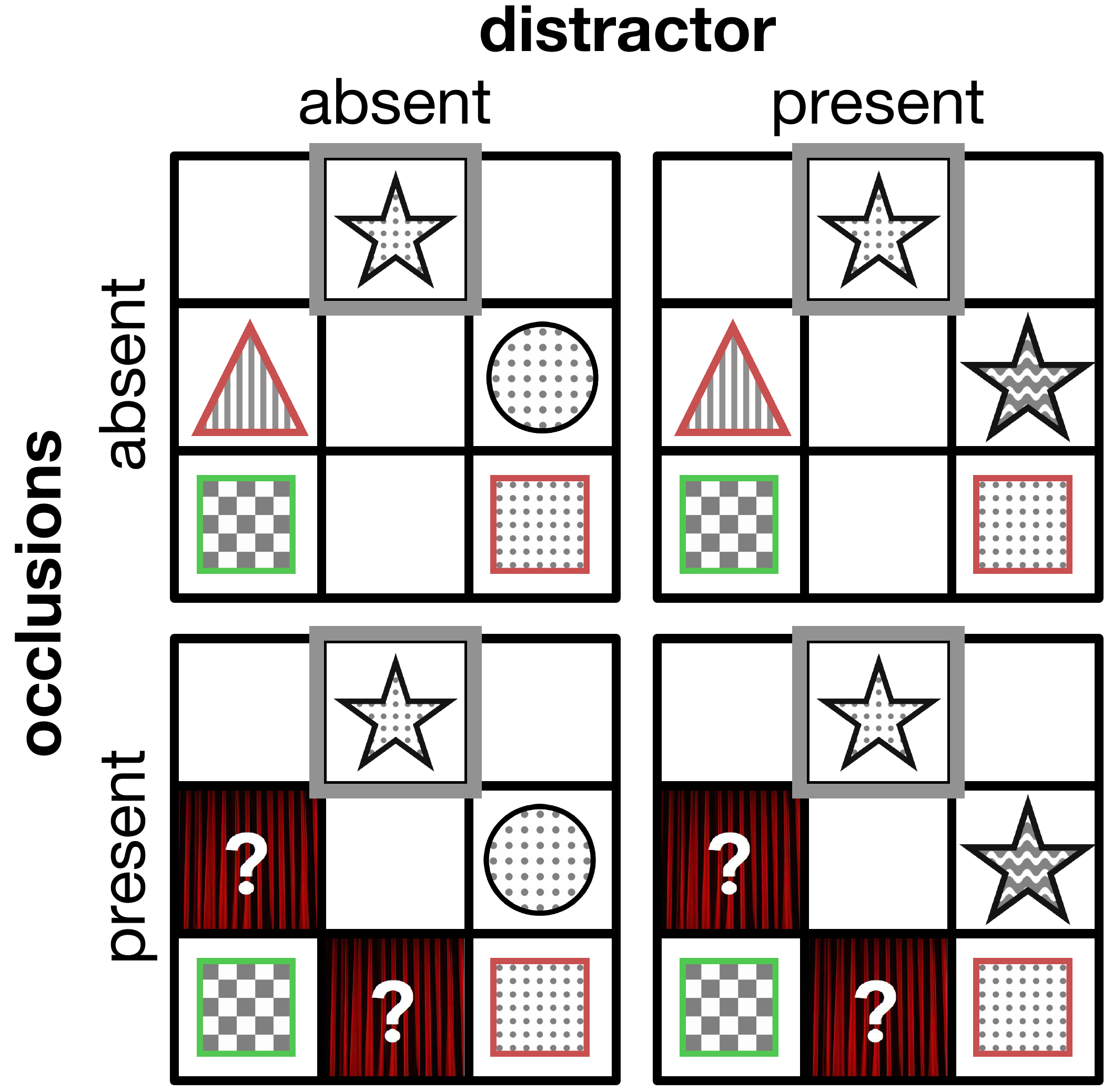}
 \caption{Stimuli in 2 $\times$ 2 design used in Experiment 1 as seen by the speaker. Grey square indicates target.}
 \label{fig:exp1design}
 \end{center}
 \end{figure}
 \subsubsection{Participants}
We recruited 102 pairs of participants from Amazon Mechanical Turk and randomly assigned speaker and listener roles. After we removed 7 games that disconnected part-way through and 12 additional games according to our pre-registered exclusion criteria (due to being non-native English speakers, reporting confusion about the instructions, or clearly violating the instructions), we were left with a sample of 83 full games. 

\subsubsection{Materials \& Procedure}
On each trial, both players were presented with a $3\times3$ grid containing objects. A \emph{target} object was privately highlighted for the speaker, who freely typed a message into a chat box in order to get the listener to click the intended referent. Participants were instructed to use visual properties of the objects rather than spatial locations in the grid. The objects varied along three discrete features (\emph{shape}, \emph{texture}, and \emph{color}), each of which took four discrete values (64 possible objects). See Appendix Fig. \ref{fig:exp1_screenshot} for a screenshot of the full interface.

There were four conditions, forming a within-pair $2 \times 2$ factorial design. 
The key manipulation was the presence or absence of occlusions (see Fig. \ref{fig:exp1design}, rows).
On `occlusion-absent' trials, all objects were seen by both participants, but on `occlusion-present' trials, two randomly selected cells of the grid were covered with occluders (curtains) from the speaker's viewpoint such that only the listener could see the contents of the cell. 
For comparison, we also included a well-studied informativity manipulation \citep[e.g.][] {pechmann1989incremental,dale1995computational,BrennanClark96_ConceptualPactsConversation, MonroeEtAl17_ColorsInContext}.
On `distractor-absent' trials, the target was the only object with a particular shape; on  `distractor-present' trials, there was a distractor with the target's shape in common ground, differing only in color or texture (see Fig. \ref{fig:exp1design}, columns).

Each trial type appeared 6 times for a total of 24 trials, and the sequence of trials was pseudo-randomized such that no trial type appeared more than twice in each block of eight trials.
On each trial, a target was randomly sampled  from the full space of objects, and a set of other objects was then randomly sampled from the remaining objects based on the condition.
On `distractor-present' trials, one of these was forced to have the same shape as the target; otherwise, they were chosen to be fillers with a different shape and randomly selected colors and textures.
To prevent the speaker from picking up on statistical patterns of the identity or quantity of hidden objects on any particular trial, we randomized the total number of ``filler'' distractors in the display (between 2 and 4) as well as the number of those distractors covered by curtains (1 or 2) on `occlusion-present' trials. 
If there were only two distractors, we did not allow both of them to be covered: there was always at least one mutually visible distractor. 
Because the distractor-present condition required the distractor with the same shape to be mutually visible, one consequence of the design was that there was never a hidden distractor with the same shape as the target.

Finally, we collected mouse-tracking data as a window into the real-time decision-making process. 
On each trial, we first asked the matcher to wait on an empty grid while the director typed their message. 
When the message was received, the matcher clicked a small circle in the center of the grid to show the objects and proceed with the trial. We recorded at 100Hz from the matcher's mouse in the decision window after this click, until the point where they started to move one of the objects. While we did not intend to analyze these data for Experiment 1, we anticipated using it in our second experiment below and wanted to use the same procedure across experiments for consistency.
 
 \subsection{Results}

Our primary measure of speaker behavior is the length (in words) of naturally produced referring expressions sent through the chat box. 
We tested differences in speaker behavior across conditions using a mixed-effect regression predicting the number of words produced on each trial. 
We included fixed effects of distractor-presence, occlusion-presence, and their interaction. 
We also included speaker-level random intercepts, as well as random effects for both slopes and the interaction\footnote{Because we randomly generated displays on each trial, there was no finite set of ``items'' with clustered data; a model adding random intercepts for the 64 target objects failed to converge.}. 

First, as a baseline, we restricted our analysis to occlusion-absent trials and examined the \emph{simple} effect of whether a distractor of the same shape as the target was present vs. absent. 
We found that speakers used significantly more words on average ($d=0.56$ words) when a distractor was present ($t = 5.6, p < 0.001$; see Fig. \ref{fig:exp1results}A). 
This replicates extensive previous findings in experimental pragmatics that speakers are sensitive to what information is needed to disambiguate different objects in common ground \cite[e.g][]{BrennanClark96_ConceptualPactsConversation,VanDeemter16_ComputationalModelsOfReferring,DaviesArnold18_HandbookReference}.

Next, we turn to the key simple effect of occlusion in `distractor-absent` contexts, which are most similar to the displays with real-world objects that we use in Experiment 2. 
We found that speakers used significantly more words on average ($d=1.25$ words) when they knew that additional objects could potentially be visible to their partner ($t = 8.8, p < 0.001$). 
Lastly, we found a significant interaction ($b = -0.49, t = 4.1, p <0.001$) where the effect of occlusion was larger in distractor-absent trials, likely reflecting a ceiling on the level of informativity required to individuate objects in our simple three-dimensional stimulus space.

 \begin{figure*}[t!]
 \begin{center}
 \includegraphics[scale=1.3]{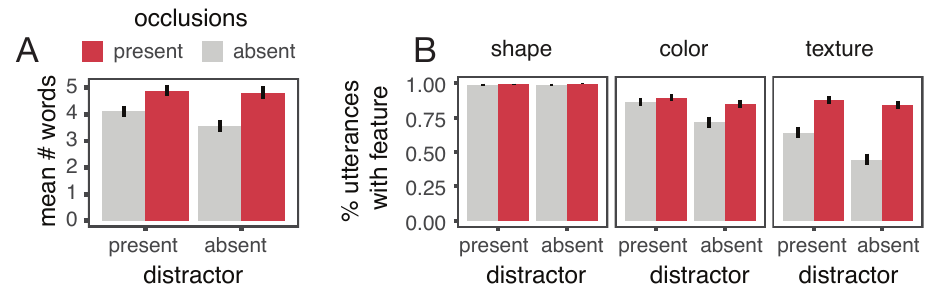}
 \caption{Results for Experiment 1. (A) Speakers used significantly more words when occlusions were present. (B) Utterances broken out by feature mentioned. Error bars on empirical data are bootstrapped 95\% confidence intervals; model error bars are 95\% credible intervals.}
 \label{fig:exp1results}
 \end{center}
 \end{figure*}

What are these additional words used for? 
As a secondary analysis, we annotated each utterance based on which of the three object features were mentioned (shape, texture, color). 
Because speakers nearly always mentioned shape  (e.g. `star', `triangle') as the head noun of their referring expression regardless of context ($\sim 99\%$ of trials), differences in utterance length across conditions must be due to differentially mentioning the other two features (color and texture). 
To test this observation, we ran separate mixed-effect logistic regressions to predict color and texture mentions.
We included fixed effects of occlusion, distractor, and their interaction. 
Due to convergence issues, we included random intercepts and slopes for each speaker, but no random interaction. 
We found simple effects of occlusion in distractor-absent contexts for both features ($b = 1.6, z = 3.2, p = 0.001$ for color; $b = 5.6, z = 6.8, p < 0.001$ for texture, see Fig. \ref{fig:exp1results}B). 
In other words, in displays like the left column of Fig. \ref{fig:exp1design} where the target was the only `star', speakers were somewhat more likely to produce the star's color---and much more likely to produce its texture---when there were occlusions present, even though shape alone is sufficient to disambiguate the target from visible distractors in both cases. 
The baseline asymmetry between production of color and texture modifiers in un-occluded contexts is consistent with prior work on over-specification \citep[e.g.][]{tarenskeen2015overspecification}.
Listener errors were rare: the target failed to be selected on only 2.5\% of trials, and we found no significant difference in error rates across the four conditions ($\chi^2(3) = 1.23, p = 0.74$). 

Finally, we inferred the speaker's probabilistic perspective weighting parameter using a quantitative Bayesian model comparison (see Appendix C for details).
We found that the inferred mixture was near the maximal endpoint allowed by our model ($w_S\approx 1$), suggesting that people's behaviors were better described by an occlusion-sensitive speaker model that considers possible hidden objects (i.e. Eq. \ref{eq:asym}), relative to an egocentric speaker model that considers only the objects in its own view (i.e., Eq. \ref{eq:ego}), or a mixture of the two. 

\subsection{Discussion}

Our results provide strong evidence supporting our model's foundational prediction that speakers increase their level of specificity in the face of occlusions.
Speakers spontaneously spent additional time and keystrokes to give further information beyond what they produced in unoccluded contexts, even though that information would be redundant given the visible objects in their own display.
The effect of occlusions on referring expressions was even larger than the classic pragmatic effect of having a similar distractor in common ground.
Critically, rather than planning their utterance purely in light of objects shared in \emph{common ground}, which was held constant across occlusion conditions, this finding shows that speakers plan their utterance relative to their uncertainty about what the \emph{listener} privately knows.

At the same time, the evidence for an intermediate perspective-taking weight was less clear in our task; rather, the inferred speaker weight was near ceiling.
One explanation for such a high level of perspective-taking is that our simplified variant of the director-matcher task was too ``easy": it did not place participants under sufficiently high cognitive load for resource considerations to play a meaningful role in their decisions about perspective-taking.
Indeed, our  resource-rational analysis in Sec.~2.4 predicted high levels of perspective-taking by both speakers and listeners in low-cost regimes.
This suspicion was further supported by the low rates of listener errors in pilot work in which we attempted to conceptually replicate  \cite{KeysarLinBarr03_LimitsOnTheoryOfMindUse} using our simplified design (see Appendix E for further details about this pilot experiment); even when faced with a confederate that deliberately produced ambiguous referring expressions (e.g. `circle' when there was one circle in the common ground and another distractor circle in the listener's private view), listeners were able to avoid selecting the occluded objects with nearly perfect accuracy. Rather than a failure to replicate the original findings, these data suggest that the simplified director-matcher task we used to test speaker predictions in Experiment 1 is not ideal for testing the further resource-rational model predictions outlined in Sec.~2.5.
Thus, in Experiment 2, we returned to the original paradigm reported by \cite{KeysarLinBarr03_LimitsOnTheoryOfMindUse} where confederate speakers were able to successfully elicit higher rates of listener errors.  

\section{Experiment~2: Manipulating speaker informativity} 

In Experiment 2, we adopted the exact stimuli and procedure used by \cite{KeysarLinBarr03_LimitsOnTheoryOfMindUse} to examine the downstream consequences of the pragmatic speaker behavior we observed in Experiment 1.
In the resource-rational framework, the deployment of effort is guided by expectations about the value of that effort: additional cost must be justified by commensurate benefits. 
Although a participant in the matcher role may begin the task with certain expectations about the director's share of the division of labor in the face of occlusions, the expected benefits of additional perspective-taking effort may shift as they obtain further evidence of the director's behavior. 
We suggest that these dynamics may help understand listener errors in prior work using the director-matcher task. 
If the confederate directors in prior work were less informative than listeners (rationally) expected at the outset, then the listener's initial allocation of perspective-taking effort may have been mis-calibrated, with detrimental consequences for their performance. 
However, our model also predicts that listeners should gradually re-adjust their effort, resulting in fewer critical errors over the course of the experiment.

We tested both of these predictions in our replication of \cite{KeysarLinBarr03_LimitsOnTheoryOfMindUse}.
In addition to a \emph{scripted} condition where speakers used the same scripted referring expressions as in the original study, we introduced a new \emph{unscripted} condition where speakers were free to generate their own referring expressions. Our goal was to replicate prior work in the scripted condition, and critically test our key prediction that listener error would decrease in the unscripted condition. More specifically, our predictions are twofold: First, a difference in listener error rate between these conditions would indicate the extent to which confederates deviated from the naturally expected division of labor: we predicted that naive speakers would spontaneously provide more informative referring expressions than confederate directors used in prior work. 
Second, a decrease in listener errors over the course of the experiment would suggest that participants are indeed able to adapt. 

\begin{figure}[b!]
 \begin{center}
 \includegraphics[scale=.4]{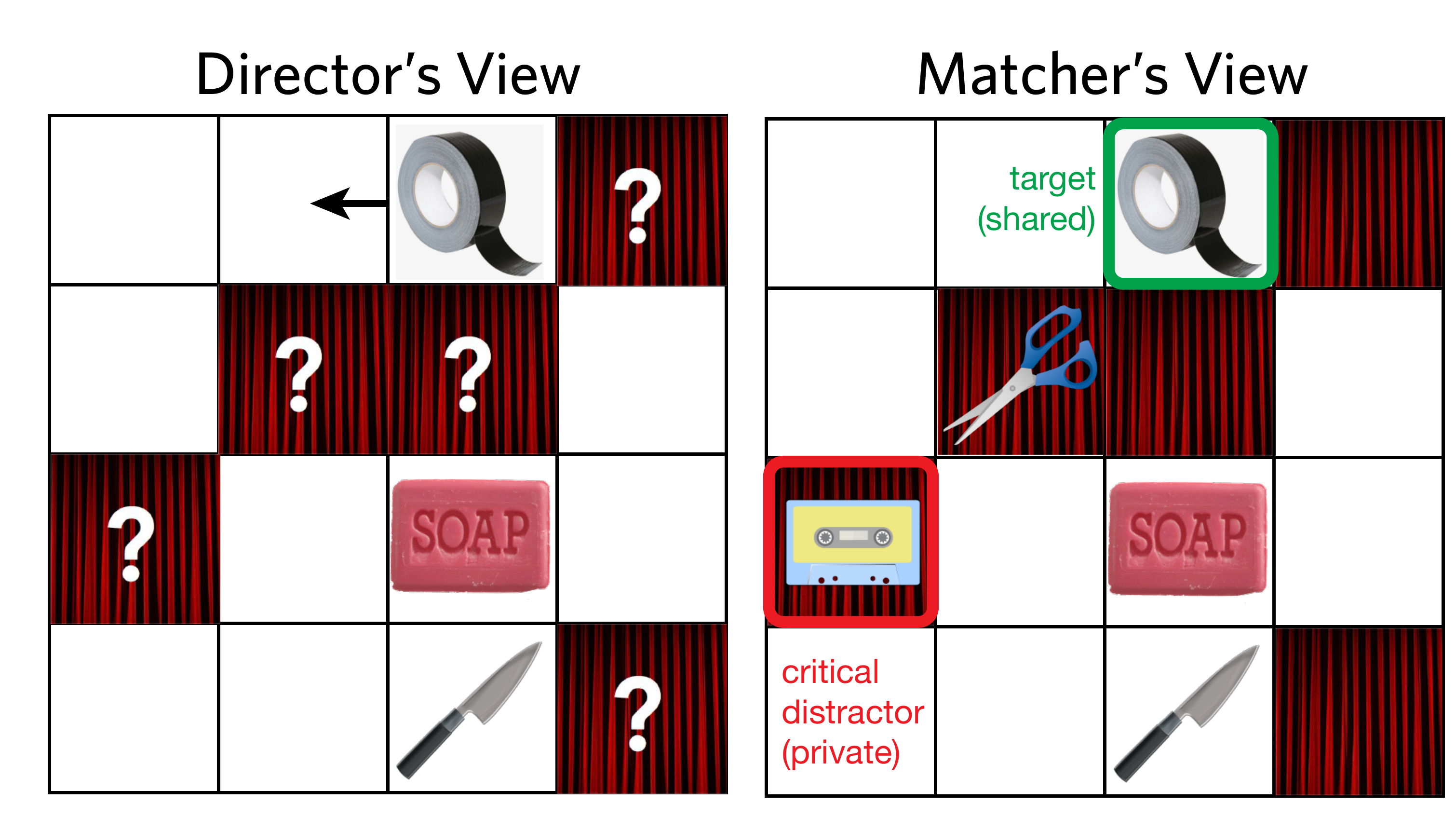}
 \caption{Critical trial of director-matcher task using the ambiguous utterance ``the tape'': a roll of tape is in view of both players, but a \emph{cassette tape} is occluded from the speaker's view.}
 \label{fig:exp2paradigm}
 \end{center}
 \end{figure}
 
\subsection{Methods}

\subsubsection{Participants}
We recruited 200 pairs of participants from Amazon Mechanical Turk. Due to a server outage, 58 pairs were unable to complete the game and were thus excluded. Following our preregistered exclusion criteria, we removed 24 pairs who reported confusion, violated our instructions, or made multiple errors on filler items, as well as 2 additional pairs containing non-native English speakers. This left 116 pairs in our final sample. 

\subsubsection{Materials and Procedure}
The materials and procedure were chosen to be as faithful as possible to those reported in \cite{KeysarLinBarr03_LimitsOnTheoryOfMindUse} while allowing for interaction over the web (we discuss the potential impact of these differences below).
Directors used a chat box to communicate where to move a privately cued target object in a $4 \times 4$ grid with five occluded cells (see Fig. \ref{fig:exp2paradigm}). 
We used exactly the same graphical representation of occlusions as in Experiment 1.
The listener then attempted to click and drag the intended object. In each of 8 objects sets, mostly containing filler objects, one target belonged to a `critical pair' of objects, such as a visible cassette tape and a hidden roll of tape that could both plausibly be called `the tape.' 

We displayed instructions to the director as a series of arrows pointing from some object to a neighboring unoccupied cell. Trials were blocked into eight sets of objects, with four instructions each.  As in \cite{KeysarLinBarr03_LimitsOnTheoryOfMindUse}, we collected baseline performance by replacing the hidden alternative (e.g.~a roll of tape) with a filler object that did not fit the critical instruction (e.g.~a battery) in half of the critical pairs. The assignment of items to conditions was randomized across participants, and the order of conditions was randomized under the constraint that the same condition would not be used on more than two consecutive items. 
All object sets, object placements, and corresponding instruction sets were fixed across participants. In case of a listener error, the object was placed back in its original position; both participants were given feedback and asked to try again. 

We used a between-subject design to compare the scripted labels used by confederate directors in prior work against what participants naturally say in the same role. For participants assigned to the director role in the `scripted' condition, a pre-scripted message using the precise wording from \cite{KeysarLinBarr03_LimitsOnTheoryOfMindUse} automatically appeared in their chat box on exactly half of trials (the 8 critical trials and about half of the fillers). Hence, the scripted condition served as a close replication. To maintain an interactive environment, we allowed the director to freely produce referring expressions on the remainder of filler trials.
In the `unscripted' condition, directors were unrestricted and free to send whatever messages they deemed appropriate on all trials, although as in Exp.~1 we explicitly asked participants not to use purely spatial descriptions (e.g. ``row 3, column 2 to row 4, column 2"). In addition to analyzing messages sent through the chat box and errors made by matchers (listeners), we collected mouse-tracking data as a window into real-time decision processes.
 
\subsection{Results}
\subsubsection{Listener errors}
 
  \begin{figure*}[t!]
 \begin{center}
 \includegraphics[scale=.9]{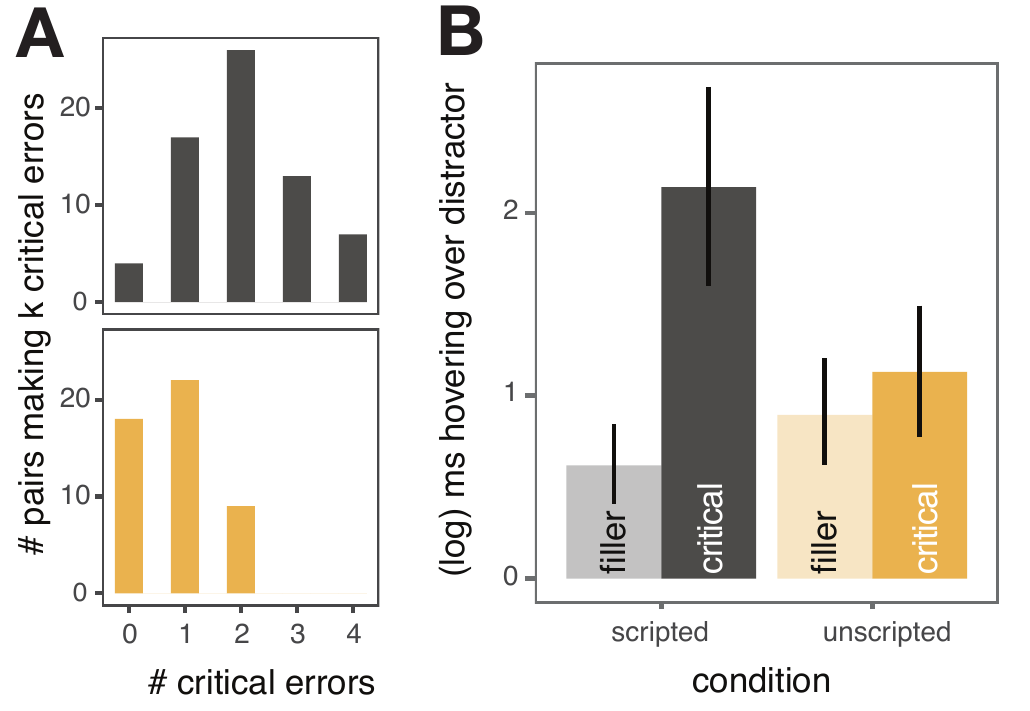}
 \caption{Listener results for Experiment 2. (A) Distribution of errors with scripted and unscripted instructions. Participants in the unscripted condition made significantly fewer errors. (B) Even when they were correct, listeners in the scripted condition were more likely to hover their mouse cursor over the distractor relative to baseline while the unscripted condition shows no difference.}
 \label{fig:exp2listener}
 \end{center}
 \end{figure*}
 
Our scripted condition successfully replicated the results of \cite{KeysarLinBarr03_LimitsOnTheoryOfMindUse} with even stronger effects: listeners incorrectly moved the hidden object on approximately 50\% of critical trials. However, on \emph{unscripted} trials, the listener error rate dropped by more than half, $p_1 = 0.51, p_2 = 0.20, \chi^2(1) = 43, p < 0.001$ (Fig. \ref{fig:exp2listener}A). While we found substantial heterogeneity in error rates across object sets (just 3 of the 8 object sets accounted for the vast majority of remaining unscripted errors; see Appendix Fig. \ref{fig:errorheterogeneity}), listeners in the unscripted condition made fewer errors for nearly every critical item. In a logistic mixed-effects model with fixed effect of condition, random intercepts for each dyad, and random slopes and intercepts for each object set,
we found a significant difference in error rates across conditions ($z = 2.6, p = 0.008$).  

It is possible that participants in the unscripted condition still \emph{considered} the hidden objects just as often as those in the scripted condition, even though they made fewer actual errors. To address this possibility, we conducted an analysis of mouse-tracking data. 
We computed the mean (log-) amount of time spent hovering over the hidden distractor and found a significant interaction between condition and the contents of the hidden cell ($t = 3.59, p <0.001$; Fig. \ref{fig:exp2listener}B) in a mixed-effects regression using dyad-level and object-level random intercepts and slopes for the difference from baseline. 
That is, while listeners in the \emph{scripted} condition spent more time hovering over the hidden cell when it contained a confusable distractor, relative to baseline (suggesting they considered the hidden object), listeners in the unscripted condition showed no difference from baseline.\footnote{Hover time was exactly zero for many trials in both conditions, which skewed the overall distribution of hover times; to address potential issues comparing the means of such zero-inflated distributions, we conducted a follow-up analysis examining the binarized \emph{proportion} of trials that listeners hovered over the hidden distractor at all, and found the same pattern of results. We also pre-registered an analysis of the latency before \emph{first} hovering over the target but due to unexpectedly poor precision in aligning response times to the beginning of the trial, we did not pursue this analysis further.}
 
\subsubsection{Adaptation over time}

Next, we examined how these error rates change over the course of the interaction. 
If the effort a listener chooses to exert depends on their expectations about the speaker's informativity, we would expect them to gradually re-calibrate their expectations through repeated observations of the speaker's behavior (see Appendix B, Fig.~7). 
That is, listeners (and speakers in unscripted interactions) may learn that the allocation of perspective-taking they initially adopted is not sufficient and flexibly adjust the extent to which they weight their partner's perspective, leading to fewer errors on later trials.
To test this hypothesis, we ran a mixed-effects logistic regression predicting whether participants made an error on critical trials as a function of the trial's position in the sequence (coded one through four).
We included random intercepts and slopes for each pair of participants, and used a fully Bayesian fitting procedure \citep{burkner2017advanced} because the random effect structure was too complex to converge using standard maximum likelihood methods.
We found a significant decrease in the probability of critical errors (i.e. attempting to move hidden objects) across both unscripted and scripted conditions ($b = 0.3,\textrm{~95\%~CI}: [0.08, 0.54]$) from an average of 43\% on the first critical trial to only 30\% on the fourth and final trial.

\subsubsection{Speaker informativity}
 
Finally, we test whether higher listener accuracy in the unscripted condition is accompanied by more informative speaker behavior than what was allowed in the scripted condition.
The simplest measure of speaker informativity is the raw number of words used in referring expressions. Compared to the scripted referring expressions, speakers in the unscripted condition used significantly more words to refer to critical objects ($b = 0.54, t = 2.6, p=0.019$  in a mixed-effects regression on difference scores using a fixed intercept and random intercepts for object and dyads). However, this is a coarse measure: for example, the shorter ``Pyrex glass'' may be more specific than ``large measuring glass'' despite using fewer words. For a more direct measure, we extracted the referring expressions generated by speakers in all critical trials and standardized spelling and grammar, yielding 122 unique labels after including scripted utterances. 

 \begin{figure*}
 \begin{center}
 \includegraphics[scale=.8]{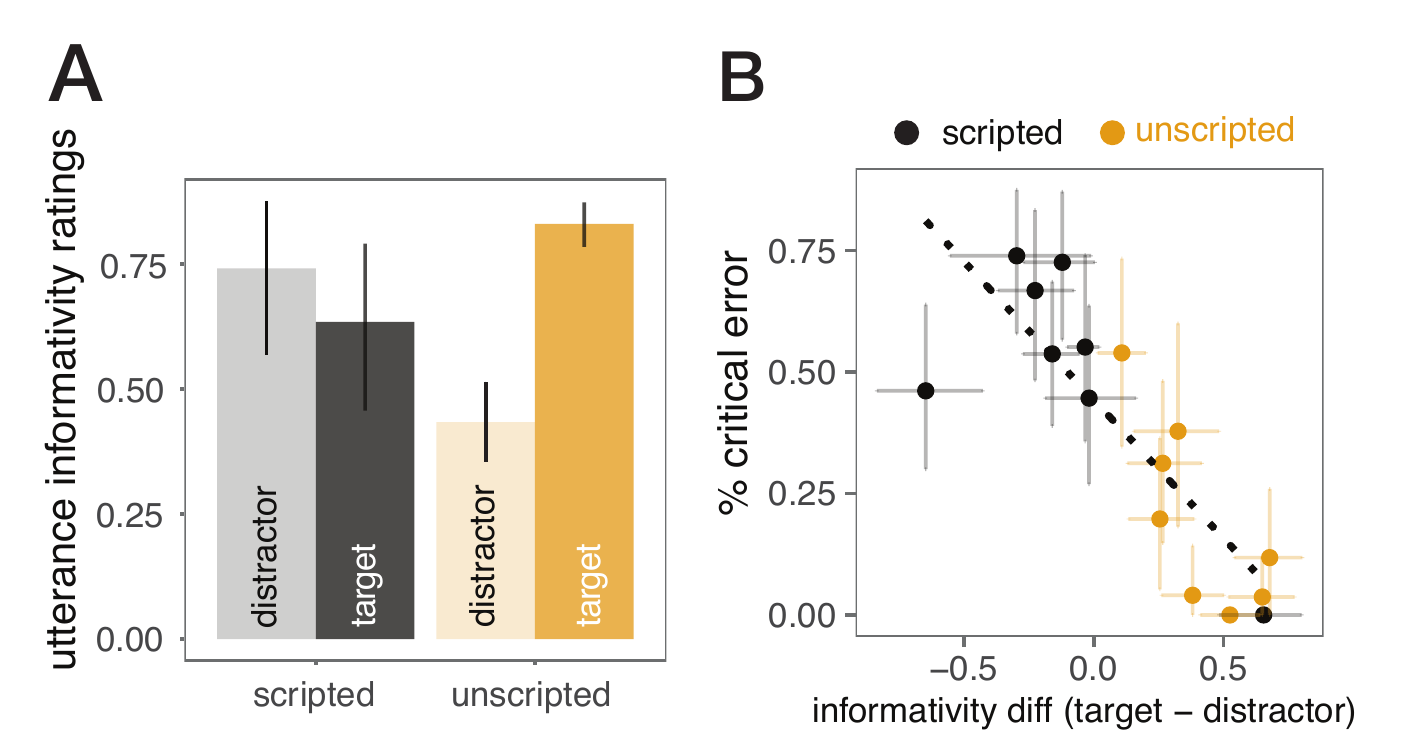}
 \caption{Speaker results for Experiment 2. (A) While speakers in the scripted condition were forced to use utterances that were judged to fit target and distractor roughly equally (by design), speakers in the unscripted condition naturally produced utterances that fit the target much better than the distractor. (B) The extent to which an utterance fits the target more than the distractor is highly predictive of error rates at an item-by-item level (dotted line is linear regression fit). All error bars are bootstrapped 95\% confidence intervals.}
 \vspace{-.75cm}
 \label{fig:exp2speaker}
 \end{center}
 \end{figure*}
 
We then recruited an independent sample of 20 judges on Amazon Mechanical Turk to rate how well each label fit the target and hidden distractor objects on a slider from ``strongly disagree'' (meaning the label ``doesn't match the object at all'') to ``strongly agree'' (meaning the label ``matches the object perfectly''). They were shown objects in the context of the full grid (with no occlusions) so that they could feasibly judge spatial or relative references like ``bottom block.'' We excluded 4 judges for guessing with response times $< 1s$. Inter-rater reliability was relatively high, with intra-class correlation coefficient of $0.54\, (95\% CI = [0.47, 0.61])$. We computed the \emph{informativity} of an utterance (the \emph{tape}) as the difference in how well it was judged to apply to the target (the cassette tape) relative to the distractor object (the roll of tape).

Our primary measure of interest is the difference in informativity across scripted and unscripted utterances. We found that speakers in the unscripted condition systematically produced more informative utterances than the scripted utterances ($d = 0.5$, 95\% bootstrapped CI = $[0.27, 0.77], p < .001$; see Appendix D for details). Scripted labels fit the hidden distractor just as well or better than the target, but unscripted labels fit the target better and the hidden distractor much worse, even though the speaker was not aware of the hidden distractor (see Fig. \ref{fig:exp2speaker}A). In other words, the scripted labels used in \cite{KeysarLinBarr03_LimitsOnTheoryOfMindUse} were less informative than expressions speakers would normally produce to refer to the same object in this context. 

These results suggest that the speaker's informativity influences listener accuracy. In support of this hypothesis, we found a strong negative correlation between informativity and error rates across items and conditions: listeners make fewer errors when utterances are a better fit for the target relative to the distractor ($\rho = -0.81$, bootstrapped 95\% CI $= [-0.9, -0.7]$; Fig. \ref{fig:exp2speaker}B). 
In other words, a large proportion of the variance in listener error rates across different items can be explained by how well utterances fit each object in their own egocentric view, consistent with a division of labor relying on higher speaker informativity.

\subsection{Discussion}

Building on Experiment 1, which aimed to identify pragmatic speaker behavior, Experiment 2 sought to test the downstream consequences of such behavior for listener perspective-taking.
More specifically, given that speakers differentially allocate effort to produce more informative utterances in the presence of occlusions, we predicted that resource-rational listeners should expect this and exert differential effort toward visual perspective-taking. 
To test this hypothesis, we used a task that has been shown to elicit high levels listener perspective-taking failure \cite{KeysarLinBarr03_LimitsOnTheoryOfMindUse}.
By comparing the utterances produced by a naive speaker to the scripted utterances produced by confederates in prior work, we found further evidence that naive speakers spontaneously produced costlier and more informative utterances, establishing a natural level of informativity that naive listeners may have expected.
Listeners, in turn, make fewer errors when playing with naive, unscripted speakers than they do playing with under-informative, scripted speakers.
Importantly, error rates decreased over the course of interaction, suggesting that even if listeners' initial expectations about the speaker's level of effort were violated, they could still adaptively increase their perspective-taking to compensate.
These findings raise several issues.

First, our use of the stimuli and procedure from \cite{KeysarLinBarr03_LimitsOnTheoryOfMindUse} successfully elicited listener errors, while our attempted conceptual replication using an under-informative confederate in the simplified Experiment 1 task did not (see Appendix E).
While there are several reasons why the simpler task may have reduced cognitive load (e.g. a smaller grid with fewer objects, fewer occlusions, a finite set of feature dimensions, and so on), it is important to emphasize the differences between the stimuli used in our two experiments, which correspond to two prominent methodological threads in the literature. 
Experiment 1 used clean property contrasts between features like color, texture, and shape, similar to the geometric stimuli used by  \cite{HannaTanenhausTrueswell03_CommonGroundPerspective} and the pure size contrasts used by \cite{HellerGrodnerTanenhaus08_Perspective}.
Experiment 2 used the much more heterogeneous items from \cite{KeysarLinBarr03_LimitsOnTheoryOfMindUse}, which included homonyms (``mouse'' for a visible stuffed animal and hidden computer device), basic-level terms for different subordinate instances (e.g. ``brush'' for a visible round-brush and a hidden flat-brush), size contrasts (e.g. ``large candle'' for a visible large candle and an even larger hidden candle),  and position contrasts (e.g. ``top block'' for a visible block on the second-to-top row and a hidden block on the top row). 

Each of these stimulus choices has its advantages and disadvantages.
On one hand, there are concerns about the generalizability of simpler variants.
Findings in narrower stimulus spaces may not straightforwardly extend to more crowded, high-variability contexts where there are not such salient and consistent dimensions along which items in each display vary. 
It is also possible that these design features of simpler variants have the effect of easing the overall cognitive load on participants.
On the other hand, the heterogeneity of the eight items from \cite{KeysarLinBarr03_LimitsOnTheoryOfMindUse} also creates serious difficulties for evaluating perspective-taking. 
We found that listener errors varied systematically across the items (see supplementary Fig. \ref{fig:errorheterogeneity}), as did the informativity of the scripted utterances, and it is challenging to place behavior across the items on the same scale, as each may be associated with distinct pragmatic considerations (e.g. relative contrast using modifiers, homonym processing, typicality of basic-level membership). 
This heterogeneity may also explain many of the remaining critical errors in the unscripted condition. 
Naive speakers often made the effort to mention multiple redundant properties given the presence of occlusions (e.g.  ``the clear audio cassette tape'' when there was only one thing that could be described as ``tape'' from their view), but because they could not know the relevant dimensions for distinguishing the target from the hidden distractors, their additional effort did not always pay off.
For example, the highest proportion of errors made in the unscripted condition occurred on the ``brush'' item, where the target and hidden distractor were so similar that almost any increase in specificity would fail to distinguish them.

This limitation also emphasizes an important consequence of the referential context.
While the relatively small number of features along which the finite stimulus space varied in Experiment 1 made it straightforward for speakers to anticipate the identity of hidden objects and provide maximally distinguishing expressions, it is computationally implausible that speakers could enumerate all possible hidden distractors in the open-ended space of objects used in Experiment 2. 
What algorithm speakers use to nevertheless produce more  informative descriptions in this open-ended space remains an open question. 
One possibility is that speakers use the distribution of \emph{visible} objects as a cue to the distribution of hidden objects, or that visible objects serve as anchors in a truncated search of semantic space. 
Another possibility is that speakers do not consider specific distractors at all and instead respond to the worst-case scenario, or use the uncertainty introduced by occlusions as a generic cue to increase their production effort along the most salient properties.

While our results closely matched those of \cite{KeysarLinBarr03_LimitsOnTheoryOfMindUse}, several key differences between the procedure of our online version and the original in-lab version prevent it from being considered a direct replication.
Most prominently, there are differences between the textual and verbal modalities with implications for the listener's processing mechanisms and the speaker's cost of production. 
Listeners in an in-lab verbal version may make eye-movements toward possible targets before the utterance has been completed, while participants in our version had to fixate on and read the message in its entirety after it had been sent. 
Conversely, we have observed in other interactive replications on the web \citep[e.g.][]{hawkins2019characterizing} that typing tends to yield shorter descriptions overall than found in the lab, suggesting that production cost in terms of effort per word may be higher for typing.
An in-lab display may also make the occlusions more natural than our virtual representation of ``curtains,'' so it is possible that the slightly greater overall number of errors we observed relative to \cite{KeysarLinBarr03_LimitsOnTheoryOfMindUse} were due to a subset of participants not understanding how the occlusions worked. However, because the same graphical representation and instructions about occlusions we were used across every condition, in both Experiments 1 and 2, these misunderstandings are unlikely to affect the comparisons of interest. 
Additionally, because we were not able to obtain the complete scripts that confederates in prior work used on \emph{filler} instructions (or even the identity of filler objects), it is possible that listeners in our scripted condition adapted to different input between critical items.
In particular, we observed that speakers in our scripted conditions used highly specific descriptions for the portion of trials on which they were allowed to freely send messages (e.g. ``the red over ear headphones'' when there was only one pair of headphones). 
These filler trials perhaps set even stronger expectations of hyper-informativity leading to larger prediction error when scripted labels were substituted in.

Finally, it is important to note that our findings do not invalidate the use of a confederate or the choice of scripted utterances in prior work; using scripted directions is a way to control the input received by the listener to study how the listener engages in perspective-taking.
Rather, our results help identify an unintended consequence of this manipulation (i.e., uncooperative directors) and clarify how it impacts listener performance. More specifically, the informativity gap between unscripted and scripted utterances highlights the role of the listener's initial expectations of speaker informativity in their allocation of effort, and how an apparent violation of these expectations may have unintended pragmatic consequences.
These expectations become especially important under higher cognitive load where the appropriate division of labor is constrained by resource-rational considerations on both sides; in such contexts, it is particularly important for both parties to consider the other's allocation of effort. While we found near-ceiling levels of speaker perspective-taking Experiment 1, the current experiment, with its relatively higher cognitive load, offered a context where speakers could not reasonably be expected to produce perfectly unambiguous utterances in this environment: even with an unscripted partner, some adaptation may be required to re-calibrate to the challenges of the context. Under the pressure for division of labor, we were able to identify clear effects of speaker informativity.

\section{General Discussion}

The longstanding debate over the role of theory of mind in communication has largely centered around the extent to which listeners (or speakers) deviate from ``optimal'' perspective-taking toward egocentric influences \citep{barr_perspective_2006,HannaTanenhausTrueswell03_CommonGroundPerspective}. 
Our work aims to present a more nuanced analysis of how resource-constrained speakers and listeners nonetheless make reasonable decisions about how to allocate their resources based on contextual expectations.
In particular, the Gricean cooperative principle emphasizes a natural division of labor in how the \emph{joint effort} of being cooperative is shared \citep{Clark96_UsingLanguage, MainwaringEtAl03_DescriptionsOfSpatialScenes}.
One important case is when the speaker has uncertainty over what the listener can \emph{see}, as in the director-matcher task. 
Our resource-rational formalization of cooperative reasoning in this context predicts that speakers (directors) naturally increase the informativity of their referring expressions to hedge against the increased risk of misunderstanding; Experiment 1 presents direct evidence in support of this hypothesis. 

Importantly, when the director is expected to contribute effort to be additionally informative, communication can be successful even when the matcher contributes less than maximal perspective-taking effort.
Indeed, the matcher will actually strike the optimal tradeoff between minimizing joint effort and maximizing communicative success by \emph{not} weighting the director's visual perspective.
This suggests a resource rational explanation of \emph{when} and \emph{why} resource-constrained listeners down-weight the speaker's visual perspective; they do so when they expect the speaker to disambiguate referents sufficiently. 
While adaptive in most natural communicative contexts, such neglect might backfire and lead to errors when the speaker (inexplicably) violates this expectation.
From this point of view, although the listener's ``failures'' may indeed be failures to identify the correct items, they are not necessarily \emph{failures of theory of mind}; rather, these inaccuracies are consistent with listeners using their theory of mind to decide when (and how much) they should expect the speaker to be cooperative and informative, and allocating their resources accordingly \citep{GriffithsLiederGoodman15_LevelsOfAnalysis}. 
Experiment 2 is consistent with this hypothesis; when speakers (directors) used underinformative scripted instructions taken from prior work, listeners made significantly more errors than when speakers were allowed to provide referring expressions at their natural level of informativity.
Furthermore, listeners were able to adapt to the speaker's level of informativity to make fewer errors over time. 

To be clear about our theoretical stance, these results do not imply that speakers are generally expected to shoulder more of the work, or that Gricean considerations free listeners to completely ignore visual perspective.
Indeed, speakers often use vague or ambiguous language that reduce their own production costs, especially when they can rely on listeners to infer the intended meaning from context \citep{ferreira2008ambiguity,wasow2015ambiguity,PiantadosiTilyGibson12_Ambiguity,peloquin2020interactions}.  
In the resource-rational elaboration of the simultaneous integration view we are advancing, the perspective-taking effort each person chooses to exert is rarely all or none: It is a matter of \emph{degree} \citep{HellerParisienStevenson16_ProbabilisticWeighing}.
There is in principle a continuum of many acceptable divisions of labor, and no single division should be considered the ``rational'' yardstick.
Instead, the resource-rational weighting for one agent should in principle depend on a number of contextual factors, including the relationship between the agents; the other agent's capacity, perspective, belief, and knowledge; the ability to avoid further clarification exchanges or repair; and the current cognitive load imposed by the environment.
It may be asymmetric when one partner is able to take on more costly processing than the other, and should be continually adjusted throughout the course of an interaction.

This flexibility is a key feature of the resource-rational framework.
An important direction for future work is to more directly explore how perspective-taking effort adjusts dynamically given aspects of the scenario \citep{GrodnerSedivy11_SpeakerSpecific, PogueEtAl16_TalkerSpecificGeneralization,ryskin_information_2019}. 
We provided preliminary evidence that, given sufficient evidence of an unusually underinformative partner, listeners may realize that devoting additional attention to which objects are occluded is necessary to maintain communicative success.
Conversely, given evidence of an overinformative partner, listeners may be able to get away with exerting less effort. 
Dynamic modulation of perspective-taking effort could be particularly functionally important in light of pervasive individual differences in working memory or executive control: variability in the capabilities of different partners should lead to variability in the appropriate division of labor, and it may not be possible to anticipate at the outset of an interaction.
Still, it is also possible that background knowledge about a partner leads to differing resource allocations even at the outset of the interaction. 
For instance, an adult may expect to shoulder more of the division of labor when interacting with a child. 

Our theoretical framework relies on an abstract computational notion of `effort' or `cost'. We remain agnostic about the precise source of these costs at the algorithmic level; the director-matcher task, like many other standard theory of mind tasks \citep{franccois2020theory}, involves the coordination of many cognitive systems, and the available data do not allow us to isolate a specific cause for poor performance  \citep{RubioFernandez16_DirectorTaskAttention}.
We expect that the abstract cost associated with using a higher mixture weight in our model represents a range of different costs associated with general executive control, working memory, selective attention, and other processes, as well as whatever cost may be specifically associated with forming and maintaining representation of a partner's likely behavior given their perspective.
For instance, it is possible that the listener can take a small number of samples from their posterior about the speaker's likely behavior and use the resulting estimate of communicative success to decide to devote less persistent attention to which cells are occluded.
If this is the case, the primary effort at stake is attentional, with the deployment of attentional resources guided by theory of mind use.
In any case, it is clear that solving the full resource-rational constrained optimization problem (Eq. 11) from scratch in every situation would be intractable: the additional effort required to compute the appropriate level of effort across these processes would exceed the resulting savings.
This has been a general challenge for resource-rational accounts, which argue that this optimization problem is solved by learning over longer (e.g. developmental) timescales \citep{lieder2019resource}; an intriguing possibility is that speakers amortize the optimization across many different partners, with relatively inexpensive adjustments based on local evidence \cite{lieder2018rational,bustamante2020learning}.
Further work in the resource-rational framework is needed to formulate explicit algorithmic theories of the ``mental labor'' associated with different processes, and how these processes are integrated to support success in communication. 

Our work also adds to the growing literature on the debate over the role of pragmatics in the director-matcher task. Recently,  \cite{RubioFernandez16_DirectorTaskAttention} has suggested that listeners monitor the speaker's level of informativity and become suspicious of the director's visual access when the director shows unexpectedly high levels of specificity in their referring expressions. Our results further bolster the argument that pervasive pragmatic reasoning about expected levels of informativity is an integral aspect of theory of mind use in the director-matcher task (and communication more generally). Note however that in this work participants became suspicious about the experimenter, while in our study participants simply adapted their expectations about informativity; a more detailed look at differences between experimental paradigms is necessary to better understand why participants drew different inferences \cite[see also][]{rubio2018joint}. Prior work also suggests that although speakers tend to be over-informative in their referring expressions \citep{KoolenGattGoudbeekKrahmer11_Overspecification,degen2019redundancy} a number of situational factors (e.g., perceptual saliency of referents) can modulate this tendency. Our work hints at an additional principle that guides speaker informativity: speakers maintain uncertainty about ``known unknowns'' in the listener's private view and may increase informativity to disambiguate the referent relative to these possible contexts.

While our experiments have focused directly on the demands of asymmetries in \emph{visual} perspective, closely following the design of \cite{KeysarLinBarr03_LimitsOnTheoryOfMindUse}, variations on this basic paradigm have also manipulated other dimensions of non-visual knowledge asymmetry, including those based on spoken information \citep{keysar_definite_1998, HannaTanenhausTrueswell03_CommonGroundPerspective}, spatial cues \citep{schober_spatial_1993, galati_flexible_2013}, private pre-training on object labels \citep{wu_effect_2007}, cultural background \citep{isaacs_references_1987}, and other task-relevant information \citep{HannaTanenhaus04_PragmaticEffects, YoonKohBrownSchmidt12_GoalsReference}. 
We expect that each of these variants introduce subtly different processing demands and pragmatic expectations, and resource rational analysis may be a useful framework for understanding how variance in these demands leads to variance in perspective-taking behavior.
Individual differences in basic cognitive function \citep[e.g.][]{ryskin_perspective-taking_2015} and the cognitive demands imposed by different tasks or environments \citep{LinKeysarEpley10_ReflexivelyMindblind} can be viewed as real differences in the underlying $\beta$ parameter, shifting the agent's decisions about perspective-taking.
Similarly, studies of how speakers \emph{inhibit} private knowledge during production may involve specific processing mechanisms involving costly executive control \citep[e.g][]{ferreira_mechanistic_2019} and resource-rational considerations may yield predictions about the extent to which private information leaks into speaker utterances \citep[see also][]{nadig_evidence_2002, heller_name_2012, BrownSchmidtTanenhaus08_TargetedGame, savitsky_closeness-communication_2011,  yoon_adjusting_2014, lane_dont_2006,wu2007effect}. 

More broadly, we suggest that a resource-rational approach may provide a more constructive and principled standard for what should constitute ``rational'' perspective-taking behavior in conversation.
As discussed by \cite{brown-schmidt_perspective_2018}, previous work arguing for egocentric heuristics has tended to use a strong classical standard of rationality. 
Any deviation from error-free perspective-taking is then taken as evidence of ``irrational'' biases that motivate a rejection of the entire rational analysis framework. 
By contrast, a more bounded standard of rationality preserves the advantages of these unifying frameworks, namely the ability to formalize the functional problem facing communicative agents at the computational level of analysis, but moves beyond the question of \emph{if} people are classically rational to ask \emph{when} and \emph{how} they make approximately optimal decisions about allocating their resources.
In other words, the resource-rational framework allows the comparison of formal proposals about which factors the agent considers when making decisions about how much perspective-taking effort to allocate, and may help to illuminate how people are so flexible across contexts.
In this way, we seek to push computational-level probabilistic weighting models toward process-level consideration of cognitive resources, forming a bridge to the initial concerns of egocentric heuristic accounts.

\section*{Acknowledgements}

This manuscript is based in part on work presented at the 38th Annual Conference of the Cognitive Science Society. An early pilot of Experiment 2 was originally conducted with input from Michael Frank and Desmond Ong. We're grateful to Victor Ferreira, Herbert Clark, and Judith Fan for thoughtful conversations and to Boaz Keysar for providing selected materials for our replication. 
Unless otherwise mentioned, all analyses and materials were preregistered at \url{https://osf.io/qwkmp/}. Code and materials for reproducing the experiment as well as all data and analysis scripts are open and available at \url{https://github.com/hawkrobe/pragmatics\_of\_perspective\_taking}.

\bibliography{refs}

\section*{Appendix A: Mathematical derivation of qualitative speaker predictions}

The key novel prediction motivating Experiment 1 is that speakers should attempt to be more informative when there is an asymmetry in visual access.
Here, we prove analytically that the predicted increase in informativity holds under fairly unrestrictive conditions. 
We define ``specificity'' extensionally, in the sense that if an utterance $u_0$ is more specific than $u_1$, then the objects for which $u_0$ is  true is a subset of the objects for which $u_1$ is true (recall that $\mathcal{L}$ is a truth-conditional semantics returning 0 or 1):
\begin{definition}
The \emph{extension} of an utterance $u$ is the set $E_u = \{ o \in \mathcal{O} | \mathcal{L}(u, o) = 1 \}$.
\end{definition}
\begin{definition}
Utterance $u_0$ is said to be \emph{more specific} than $u_1$ iff $E_{u_0} \subset E_{u_1}$, where we define $\mathcal{O}^* = E_{u_1} \setminus E_{u_0}$.
\end{definition}

We now show that our ``ideal'' recursive reasoning model predicts that speakers should prefer more informative utterances in contexts with occlusions. In other words, that the \emph{asymmetry} utility leads to a preference for more specific referring expressions than the \emph{egocentric} utility. 

\begin{theorem}
If $u_0$ is more specific than $u_1$ then the following holds for any target $o^t$ and shared context $C$: 
$$
\frac{S_{asym}(u_0 | o^t, C)}{S_{asym}(u_1| o^t, C)}
> 
\frac{S_{ego}(u_0 | o^t, C)}{S_{ego}(u_1 | o^t, C)}
$$

\end{theorem}

\begin{proof}
Since $S(u_0|o^t, C)/S(u_1|o^t, C) =  \exp(\alpha\cdot(U(u_0; o^t, C) - U(u_1;o^t,C)))$ it is sufficient to show 
$$
U_{asym}(u_0 ; o, C) - U_{asym}(u_1; o, C)
> 
U_{ego}(u_0 ; o, C) - U_{ego}(u_1 ; o, C)
$$
We first break apart the sum on the left-hand side:

\begin{eqnarray}
 U_{asym}(u_0 | o^t, C) - U_{asym}(u_1 | o^t, C)
&=& \displaystyle\sum_{o_h \in \mathcal{O}} p(o_h)\left[\log L(o | u_0, C\cup o_h) - \log L(o|u_1, C \cup o_h)\right]  \notag\\
  & = & \displaystyle\sum_{o^*\in\mathcal{O}^*} p(o^*) \log\frac{L(o^t|u_0, C\cup o^*)}{L(o^t|u_1, C\cup o^*)} \label{eqn:term1}\\ 
 & & + \displaystyle\sum_{o_h\in\mathcal{O}\setminus\mathcal{O}^*} p(o_h) \log \frac{L(o^t|u_0, C\cup o_h)}{L(o^t|u_1, C\cup o_h)} \label{eqn:term2}
\end{eqnarray}

By the definition of  $\mathcal{O^*}$ we have $\mathcal{L}(u_0, o_h) = \mathcal{L}(u_1, o_h)$ for objects $o_h$ in the complement $\mathcal{O} \setminus \mathcal{O^*}$. Therefore, for \ref{eqn:term2}, $L(o^t | u_i, C \cup o_h) = L(o^t | u_i, C)$, giving us $\log\frac{L(o^t | u_0, C)}{L(o^t|u_1, C)}\sum_{o_h\in\mathcal{O}\setminus\mathcal{O}^*}p(o_h)$ 

For the ratio in \ref{eqn:term1}, we can substitute the definition of the listener $L$ and simplify: 
$$
\begin{array}{rcl}
\displaystyle\frac{L(o^t|u_0, C\cup o^*)}{L(o^t|u_1, C\cup o^*)}  
& = & \displaystyle\frac{\mathcal{L}(o^t, u_0) [\sum_{o\in C \cup o^*}\mathcal{L}(o,u_1)]}{\mathcal{L}(o^t, u_1) [\sum_{o\in C \cup o^*}\mathcal{L}(o,u_0)]} \\[.5cm]
& = & \displaystyle\frac{\mathcal{L}(o^t, u_0) [\mathcal{L}(o^*, u_1) + \sum_{o\in C}\mathcal{L}(o,u_1)]}{\mathcal{L}(o^t, u_1) [\mathcal{L}(o^*, u_0)  + \sum_{o\in C}\mathcal{L}(o,u_0)]} \\[.5cm]
& < & \displaystyle\frac{\mathcal{L}(o^t, u_0) [\sum_{o\in C}\mathcal{L}(o,u_1)]}{\mathcal{L}(o^t, u_1) [\sum_{o\in C}\mathcal{L}(o,u_0)]} \\[.5cm]
& = & \displaystyle\frac{L(o^t|u_0, C)}{L(o^t|u_1, C)}  
\end{array}
$$

Thus, 

$$
\begin{array}{rcl}
U_{asym}(u_0 | o^t, C) - U_{asym}(u_1 | o^t, C) & < & \log\frac{L(o^t | u_0, C)}{L(o^t|u_1, C)}\left(\displaystyle\sum_{o^*\in\mathcal{O}^*}p(o^*) + \displaystyle\sum_{o_h\in\mathcal{O}\setminus\mathcal{O}^*}p(o_h)\right) \\
&=& \log L(o^t | u_0, C) - \log L(o^t | u_1, C) \\
&=& U_{ego}(u_0 | o^t, C) - U_{ego}(u_1 | o^t, C)
\end{array}
$$
\end{proof}

Note that this proof also holds when an utterance-level cost term $\textrm{cost}(u)$ penalizing longer or more effortful utterances is incorporated into the utilities
$$
\begin{array}{lcl}
U_{asym}(u; o, C_s) & = &  \sum_{o_h \in \mathcal{O}} \log L_0(o | u, C_s \cup o_h)P(o_h) - \textrm{cost}(u) \\
U_{ego}(u; o, C) & = & \log L(o | u, C) - \textrm{cost}(u)
\end{array} 
$$
since the same constant appears on both sides of inequality. 
It also follows that a speaker using any mixture of the asymmetric and egocentric utilities (i.e. $w_SU_{ego} + (1-w_S)U_{asym}$ where $w_S > 0$) will monotonically prefer more informative utterances than a purely egocentric speaker.

\section*{Appendix B: Model prediction for flexibility over extended interaction}

  \begin{figure*}[b!]
 \begin{center}
 \includegraphics[scale=1]{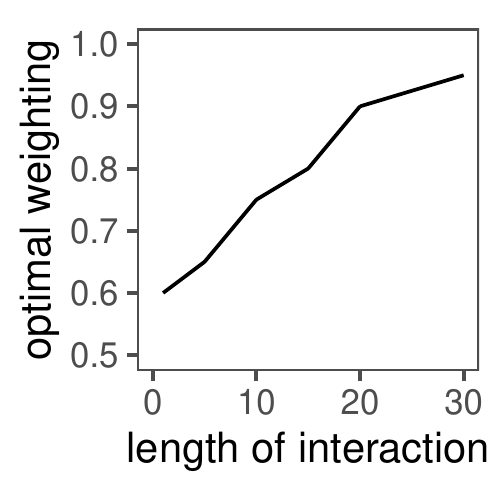}
 \caption{Our model predicts that the listener should flexibly increase the effort they dedicate to perspective-taking as they infer from the speaker's short utterances that the speaker is dedicating less effort. For these simulations, we set $\beta = 0.1$.}
 \label{fig:dynamics}
 \end{center}
 \end{figure*}
 
Another key prediction that distinguishes a resource-rational framework from a ``fixed capacity'' egocentric heuristic account is that agents may flexibly adjust the effort dedicated to perspective-taking depending on contextual factors.
In this section, we derive the prediction that listeners adapt their perspective-weighting over the course of several rounds where the speaker is less informative than initially expected. 
The basic mechanism for this adaptation in our model is an inference about the underlying perspective-taking weighting being used by the speaker based on data.
Because the speaker is expected to behave differently under different settings of the parameter $w_S$, data $D = \{(u, o)\}_i$ from repeated observations of the speaker's choice of utterance $u$ for targets $o$ provides a statistical signal about which $w_S$ they are likely to be using.
Using Bayes rule, the posterior over $w_S$ is given by $D$:

\begin{equation}
\begin{array}{rcl}
P(w_S | D) &  \propto & P(D | w_S)P(w_S)  \\
 & = & P(w_S)\cdot \prod_i P_{S_1}(u_i; o_i, C, w_S) 
\end{array}
\end{equation}

We now conduct a resource-rational analysis of a listener using this posterior instead of the uniform prior $P(w_S)$.
Specifically, we use the posterior after observing the speaker provide a single-word utterance to refer to the target over a specified number of rounds.
Note that this single-word utterance is completely sufficient to distinguish the target given the objects in common ground, so it is only ``under-informative'' relative to what we previously established a perspective-taking speaker would do to account for the fact the listener may see hidden objects they do not.
Results are shown in Fig. \ref{fig:dynamics}.
As the listener observes more and more evidence that the speaker is exerting a low level of perspective-taking effort, the boundedly optimal setting of their own perspective-taking effort grows higher.
In other words, the division of communicative labor gradually shifts onto the listener to preserve communicative success.

\section*{Appendix C: Quantitative model comparison for Experiment 1}

In this section, we conduct a quantitative model comparison using our empirical data to further bolster the qualitative speaker predictions derived in the previous section. Specifically, we describe the details of a Bayesian Data Analysis evaluating our mixture model on the empirical data, and comparing it to the purely egocentric (or ``occlusion-blind'') baseline model (Eq. \ref{eq:ego}), which does not reason about the possible existence of hidden objects behind occlusions.

The implementation of the director-matcher task for the model was the same as we used for the resource-rational simulations presented in Appendix A. 
Because there were no differences observed in production based on the particular levels of target features (e.g. whether the target was blue or red), we again collapsed across these details and only provided the model which features of each distractor differed from the target on each trial. 
After this simplification, there were 4 possible kinds of contexts: \emph{distractor-absent} contexts, where the other objects differed in every dimension, and three varieties of \emph{distractor-present} contexts, where the critical distractor differed in \emph{only shape}, \emph{shape and color}, or \emph{shape and texture}. In addition, we provided the model information about whether each trial had cells occluded or not. 
The space of utterances for the speaker model was derived from our feature annotations: for each trial, the speaker model selected among 7 utterances referring to each combination of features: only mentioning the target's shape, only mentioning the target's color, mentioning the shape \emph{and} the color, and so on. For the set of alternative objects $\mathcal{O}$ that the speaker marginalizes over, we used a uniform prior over all combinations of sharing the same or different properties as the target (i.e. the same as the possible distractors).

Our full mixture model has five free parameters which we infer from the data using Bayesian inference\footnote{Note that this use of Bayesian statistics in analyzing and evaluating our cognitive model is completely dissociable from the assumption of Bayesian recursive reasoning within the model.}. 
The speaker optimality parameter, $\alpha$, is a soft-max temperature such that at $\alpha = 1$, the speaker produces utterances directly proportional to their utility, and as $\alpha \rightarrow \infty$ the speaker shifts to maximizing. In addition, to allow for the differential production of the three features (i.e. Fig. \ref{fig:exp1results}B), we assume separate production \emph{costs} for each feature: a texture cost $c_t$, a color cost $c_c$, and a shape cost $c_s$. Finally, we also fit the mixture weight $w_S$.
We use (uninformative) uniform priors for all parameters: 
$$
\begin{array}{rcl}
\alpha & \sim & \textrm{Unif}(0, 1000) \\
w_S & \sim & \textrm{Unif}(0,1) \\
c_t, c_c, c_s & \sim & e^{\textrm{Unif}(-10,1)}
\end{array}
$$

  \begin{figure*}[t!]
 \begin{center}
 \includegraphics[scale=.75]{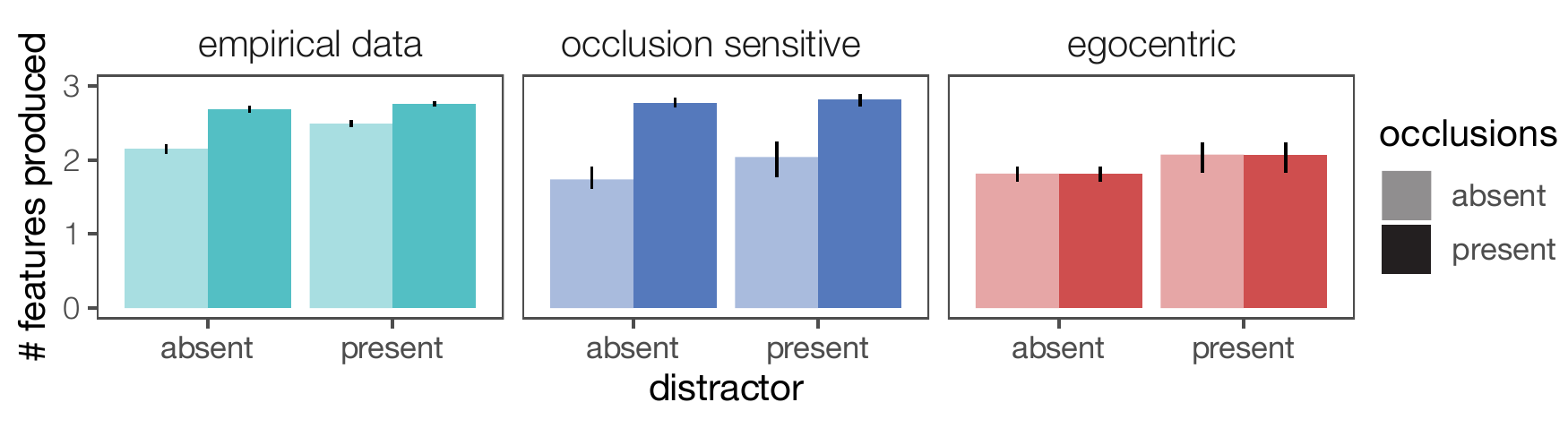}
 \caption{Quantitative modeling results for Experiment 1. Posterior predictives of each model are projected to the mean number of features produced in each condition. Error bars on empirical data are bootstrapped 95\% confidence intervals; model error bars are 95\% credible intervals.}
 \label{fig:exp1model}
 \end{center}
 \end{figure*}
We obtained predictions from our speaker model (i.e. a distribution over the possible utterances) for a particular setting of parameters using analytic enumeration. 
These predictions were mixed with a 5\% chance of randomly guessing to obtain a likelihood function for scoring the empirical data.
Finally, we obtained a posterior over parameters using MCMC. 
We discarded 1000 burn-in samples and then drew 1000 samples from the posterior with a lag of 1.  
Posterior predictives were computed by sampling parameters from these posteriors and taking the expected number of features produced by the speaker, marginalizing over possible non-critical distractors in context  (this captures the statistics of our experimental contexts, where there was always a distractor sharing the same color or texture but a different shape as the target). 
Finally, to obtain marginal likelihoods for a model comparison, we averaged 39 runs of annealed importance sampling (AIS) for each model, taking 10,000 steps per run. 
We implemented our models and conducted inference in the probabilistic programming language WebPPL (Goodman \& Stuhlmuller, 2014). All code necessary to reproduce our model results is available at the project github: \textrm{https://github.com/hawkrobe/pragmatics\_of\_perspective\_taking}.

\begin{table}[]
\centering
\begin{tabular}{|l|l|}
\hline
\textbf{model}          & \textbf{marginal likelihood} \\ \hline
egocentric ($w_S=0$)          & -3836 \\ \hline
occlusion-sensitive ($w_S=1$) & \textbf{-3038} \\ \hline
mixture                       & -3172 \\ \hline
\end{tabular}
\caption{Model comparison conditioned on Experiment 1 data.}
\end{table}

Our primary model comparison is to compare the full mixture model to the endpoints, with $w_S=0$ corresponding to a purely egocentric or ``occlusion-blind'' speaker, and $w_S = 1$ corresponding to our occlusion-sensitive speaker.
First, we found extremely strong support for the pure  occlusion-sensitive model relative to the pure occlusion-blind model, providing quantitative backing to the qualitative failure of an egocentric model to predict differences between occlusion-present and occlusion-absent trials.
Somewhat surprisingly, however, we also found support for the pure occlusion-sensitive speaker over the mixture model: the Bayesian Occam's razor determined that the additional model complexity contributed by the mixture parameter was not justified by sufficient increases in predictive accuracy and prefers the simpler model.
This result, along with the corresponding listener results reported in Appendix E, suggests that the simplified variant of the director-matcher task used in Experiment 1 may not be sufficiently cognitively demanding to elicit (resource-rational) failures of perspective-taking in either speakers or listeners, and may correspond to the optimal levels of perspective-taking predicted at lower levels of perspective-taking cost $\beta$ (see Fig. \ref{fig:speaker_sim}).

Next, to examine the pattern of behavior of each model, we computed the posterior predictive on the expected number of features mentioned in each trial type of our design. While the occlusion-blind speaker model successfully captured the simple effect of distractor-absent vs. distractor-present contexts, it failed to account for behavior in the presence of occlusions. The occlusion-sensitive model, on the other hand, accurately accounted for the full pattern of results (see Fig \ref{fig:exp1model}). Finally, we examined parameter posteriors for the best-fitting occlusion-sensitive model with $w_S=1$ (see Appendix Fig. \ref{fig:paramposterior}): the inferred production cost for \emph{texture} was significantly higher than that for the other features, accounting for why participants were overall less likely to include texture in their descriptions relative to color.

\section*{Appendix D: Multi-stage bootstrap procedure for Experiment 2}

The statistical dependency structure of our ratings was more complex than standard mixed-effect model packages are designed to handle and the summary statistic we needed for our test was a simple difference score across conditions, so we instead implemented a custom multi-stage, non-parametric bootstrap scheme to appropriately account for different sources of variance. In particular, we needed to control for effects of \emph{judge}, \emph{item}, and \emph{speaker}.

First, to control for the repeated measurements of each judge rating the informativity of all labels, we resampled our set of sixteen \emph{judge} ids with replacement. For each label, we then computed informativity as the difference between the target and distractor fits within every judge's ratings, and took the mean across our bootstrapped sample of judges. Next, we controlled for item effects by resampling our eight \emph{item} ids with replacement. Finally, we resampled \emph{speakers} from pairs within each condition (scripted vs. unscripted), and looked up the mean informativity of each utterance they produced for each of the resampled set of items. Now, we can take the mean within each condition and compute the difference across conditions, which is our desired test statistic. We repeated this multi-stage resampling procedure 1000 times to get the bootstrapped distribution of our test statistic that we reported in the main text. Individual errors bars in Fig. 4 are derived from the same procedure but without taking difference scores. 

\section*{Appendix E: Supplemental experiment}

To further motivate our rationale for using the original materials and design from Keysar et al (2003) in Experiment 2, we conducted a version of the same listener manipulation using the stimuli from Experiment 1\footnote{We are grateful to an anonymous reviewer for suggesting this experiment.}.
We recruited $N=72$ participants on Amazon Mechanical Turk and placed them into the same environment used in Experiment 1 with several key changes to the trial sequence.
First, we removed the occlusion-absent condition, so every trial contained occlusions, generated randomly on each trial to cover two cells.
Second, in every block of eight trials, we included two ``critical trials'' where we placed an occluded distractor in the listener's private view with the same shape as the target.
Third, we added a ``practice'' block of four non-critical trials at the front of the trial sequence, leading to a total of 28 trials.
Otherwise, the experiment design and stimuli were held constant. 

Instead of recruiting real speakers for real-time, multi-player interaction, as in Experiments 1 and 2, we used a simple bot as our scripted confederate.
On critical trials, it produced an ambiguous utterance mentioning only the shape (e.g. ``the square''). 
When an object with the same shape as the target appeared in common ground, it would produce an utterance mentioning a perfectly distinguishing attribute (e.g. ``the blue square'' if there were no other blue objects) or produce an exhaustive three-word utterance if distractors existed on each dimension.
Otherwise, to prevent short utterances from being suspicious, it  produced shape-only utterances on two-thirds of filler trials, and added an additional modifier on the other one-third. 

As in Experiment 2, our primary measure is the proportion of errors on critical trials. 
Unlike in Experiment 2, we found no evidence that errors on critical trials, requiring the use of theory of mind, were higher than on filler trials. 
Excluding practice trials, we found an error rate of 4.9\% on critical trials and an error rate of 8.4\% on filler trials.
If anything, we find that the error rate on critical trials was significantly \emph{lower} than on filler trials $\chi^2(1)=5.9, p = 0.015$. 
When we implement the strict exclusion criterion used in Experiment 2, excluding $N=25$ participants who made more than one error on filler trials (under the rationale that these participants may be generally unattentive), we find that only 9 of the remaining 49 participants made any critical errors at all, at any point in the experiment, and the error rate was still not significantly higher than the error rate on filler items ($4.6$\% for critical trials, $3.3$\% for filler items, $\chi^2(1)=1.02, p = 0.312$).
Under both analyses, the prevalence of errors was dramatically lower than reported by Keysar et al (2003) or in our Experiment 2, using the Keysar stimuli. The presence of this ceiling effect suggests that this simple stimulus space may not be sufficiently cognitively demanding for listeners (due to a variety of possible design factors) to allow us to ask more detailed questions about failures of perspective-taking, so we did not proceed to run the corresponding ‘unscripted’ condition. 

\begin{figure}[h]
\centering
\includegraphics[]{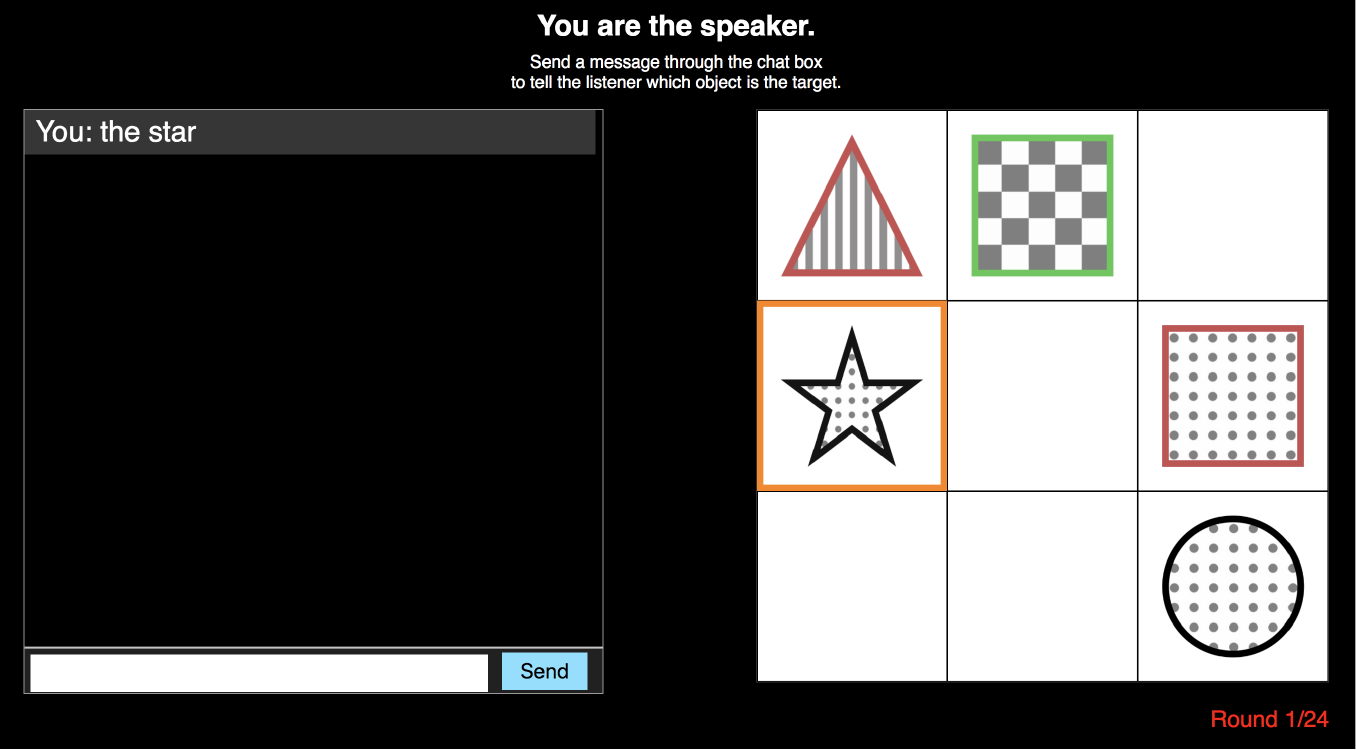}
\caption{Screenshot of experiment interface.}
\label{fig:exp1_screenshot}
\end{figure}


\begin{figure}[h]
\centering
\includegraphics[scale=0.9]{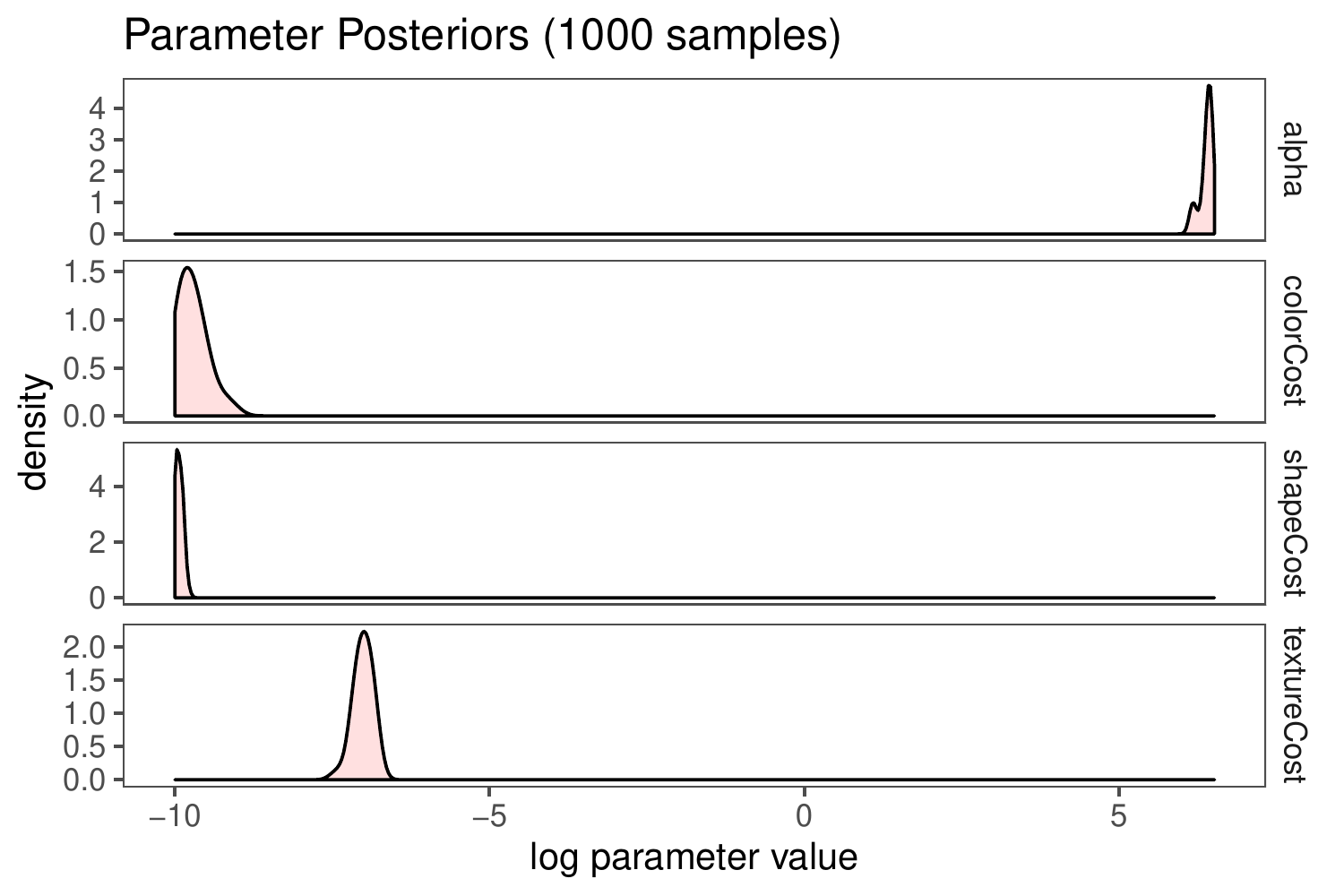}
\caption{Supplementary figure of parameter posteriors. All parameters shown on log scale. MAP estimates with 95\% highest posterior density intervals are as follows: $\alpha= 537, HDI = [498, 593]$; $c_{color} = 4.99 \times 10^{-5}, HDI = [4.5 \times 10^{-5}, 8.5 \times 10^{-5}]$; $c_{shape} = 4.7 \times 10^{-5}, HDI = [4.5 \times 10^{-5}, 5.9 \times 10^{-5}]$; $c_{texture} = 1.01 \times 10^{-3}, HDI = [7.4 \times 10^{-4}, 1.19 \times 10^{-3}]$ }
\label{fig:paramposterior}
\end{figure}

\begin{figure}[h]
\centering
\includegraphics[]{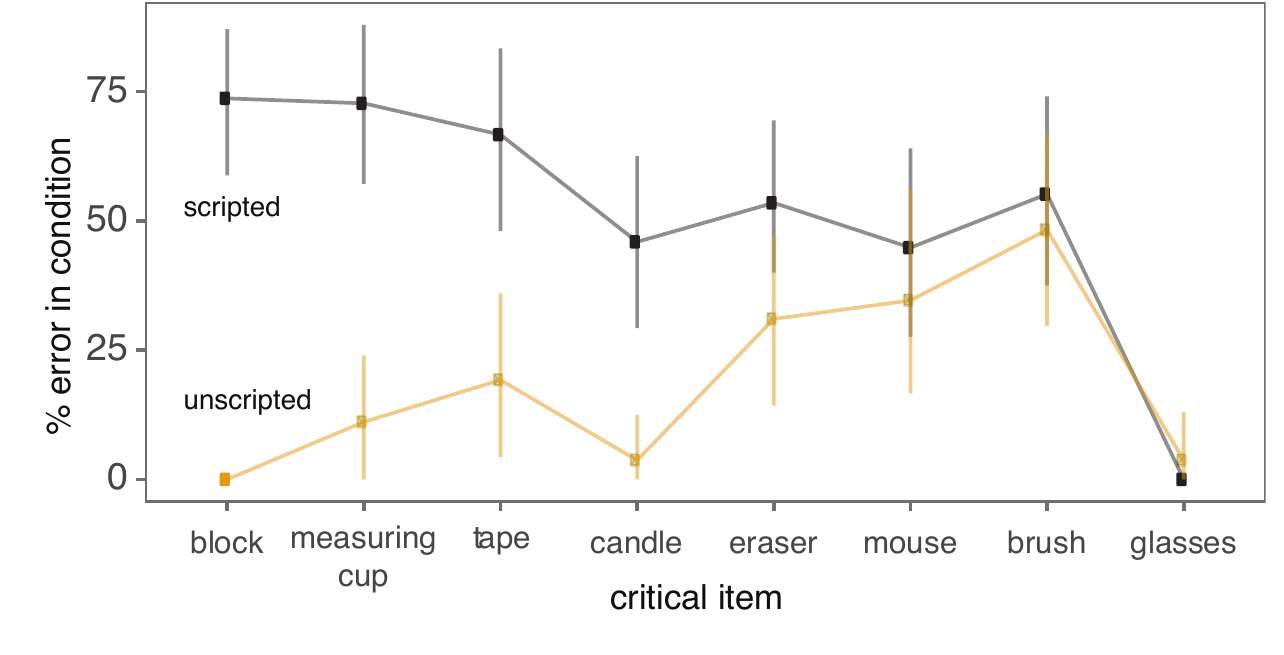}
\caption{Supplementary figure of heterogeneity in errors across the 8 object sets used in Experiment 2 (from Keysar, 2003). Error rates across object diverge significantly from a uniform distribution in both scripted ($\chi^2=55, p < 0.001$) and unscripted ($\chi^2=36, p <0.001$) conditions under a non-parametric $\chi^2$ test.}
\label{fig:errorheterogeneity}
\end{figure}

\end{document}